\documentclass{article}

\PassOptionsToPackage{sort, numbers, compress}{natbib}

\usepackage[preprint]{neurips_2021}

\usepackage{graphicx}
\usepackage{amssymb,amsmath}
\usepackage{latexsym}
\usepackage{pst-node,pst-tree}
\usepackage{xcolor}
\usepackage{algorithmic}
\usepackage[linesnumbered,ruled,vlined]{algorithm2e}
\usepackage{bm}
\usepackage{setspace}
\usepackage{subfig}
\usepackage{booktabs}
\usepackage{amsthm}
\usepackage{url}
\usepackage[version=4]{mhchem}
\usepackage{wrapfig}
\usepackage{textcomp}
\usepackage{changepage}
\usepackage{caption}
\usepackage{romannum}
\usepackage{tikz}
\usepackage{cancel}
\usepackage{enumitem}
\usepackage{mathtools}
\usepackage{float}
\usepackage{dsfont}
\usepackage{titletoc}
\usepackage{upgreek}

\usepackage[utf8]{inputenc} \usepackage[T1]{fontenc}    \usepackage{hyperref}       \usepackage{amsfonts}       \usepackage{nicefrac}       \usepackage{microtype}

\definecolor{scarlet}{RGB}{190, 1, 25}
\definecolor{crimson}{RGB}{153, 0, 0}
\definecolor{waterblue}{RGB}{55, 120, 191}
\definecolor{tangerine}{RGB}{249, 115, 6}
\definecolor{grassgreen}{RGB}{77, 164, 9}

\renewcommand{\vec}[1]{\bm{\mathbf{#1}}}
\newcommand{\inv}[1]{\frac{1}{#1}}
 \DeclareMathOperator*{\argmax}{argmax} \newcommand{\img}{\mathrm{im}}
\newcommand{\scx}{\mathrm{SC}}
\newcommand{\diag}{\mathrm{diag}}

\newcommand{\pequal}{\mathrel{\phantom{=}}}
\newcommand{\rrr}{\mathbb R}

\newcommand{\dd}{\mathsf{d}}

\newcommand{\LL}{\vec{\mathcal L}}
\newcommand{\M}{{\mathcal M}}    
\newcommand{\EE}{\vec{\mathcal E}}
\newcommand{\cub}{\mathrm{CB}}
\newcommand{\difflap}{\mathsf{DiffL}}

\newcommand{\dattwoholes}{\texttt{PUNCTPLANE}}
\newcommand{\dattorus}{\texttt{TORUS}}
\newcommand{\datthreetorus}{\texttt{3-TORUS}}
\newcommand{\datgenustwo}{\texttt{GENUS-2}}
\newcommand{\datgenusfour}{\texttt{TORI-CONCAT}}
\newcommand{\datethanol}{\texttt{ETH}}
\newcommand{\datmda}{\texttt{MDA}}
\newcommand{\datrna}{\texttt{PANCREAS}}
\newcommand{\datbudd}{\texttt{3D-GRAPH}}
\newcommand{\datisland}{\texttt{ISLAND}}
\newcommand{\datretina}{\texttt{RETINA}}

\newcommand{\bit}{\begin{itemize}}
\newcommand{\eit}{\end{itemize}}
\newcommand{\benum}{\begin{enumerate}}
\newcommand{\eenum}{\end{enumerate}}
\newcommand{\beq}{\begin{equation}}
\newcommand{\eeq}{\end{equation}}
\newcommand{\paragraphbf}[1]{\textbf{#1\enskip}}

\theoremstyle{definition}
\newtheorem{definition}{Definition}
\newtheorem*{remark}{Remark}

\theoremstyle{plain}

\newtheorem{theorem}[definition]{Theorem}
\newtheorem{corollary}[definition]{Corollary}
\newtheorem{lemma}[definition]{Lemma}
\newtheorem{proposition}[definition]{Proposition}

\newtheorem{assumption}{Assumption}

\renewenvironment{proof}{\paragraph{Proof.}}{\hfill\qed}

\newenvironment{proofof}[1]{\paragraph{Proof of #1.}}{\hfill\qed}

\newcommand{\suppnumberprefix}{S}
\newcommand{\setupsupp}{
	\clearpage
	\newpage
	\appendix

	\renewcommand{\thefigure}{\suppnumberprefix\arabic{figure}}
	\setcounter{figure}{0}

	\renewcommand{\thetable}{\suppnumberprefix\arabic{table}}
	\setcounter{table}{0}

	\renewcommand{\theequation}{\suppnumberprefix\arabic{equation}}
	\setcounter{equation}{0}

	\renewcommand{\thealgocf}{\suppnumberprefix\arabic{algocf}}
	\setcounter{algocf}{0}

	\renewcommand{\thedefinition}{\suppnumberprefix\arabic{definition}}
	\setcounter{definition}{0}

	\renewcommand{\theassumption}{\suppnumberprefix\arabic{assumption}}
	\setcounter{assumption}{0}

	\section*{\centering Supplementary Material of \\ \titlename}

	\pagenumbering{roman}
	\setcounter{page}{1}
	\section*{Table of Contents}
	\startcontents[appendix]
	\printcontents[appendix]{l}{1}{\setcounter{tocdepth}{2}}
	\newpage

\pagenumbering{arabic}
	\setcounter{page}{1}
}
\newcommand{\suppname}{Supplement}

\newcommand{\usedbibstyle}{alpha}

\usepackage{expl3}
\usepackage{xparse}
\ExplSyntaxOn
\prop_new:N \g_leal_tooltip_prop
\NewDocumentCommand{\CSet}{mm}
 {
  \prop_gput:Nnn \g_leal_tooltip_prop {#1} {#2}
 }
\NewDocumentCommand{\CGet}{m}
 {
  \prop_item:Nn \g_leal_tooltip_prop {#1}
 }
\ExplSyntaxOff

\newcommand{\figref}[1]{\CGet{#1}}

\CSet{fig:synth-datasets}{\ref{fig:synth-datasets}}

\CSet{fig:hemb-2-loops}{\figref{fig:synth-datasets}a}
\CSet{fig:detected-cycles-2-loops}{\figref{fig:synth-datasets}b}

\CSet{fig:detected-cycles-2-torus}{\figref{fig:synth-datasets}c}
\CSet{fig:hemb-2-torus}{\figref{fig:synth-datasets}d}
\CSet{fig:detected-cycles-3-torus}{\figref{fig:synth-datasets}e}
\CSet{fig:hemb-3-torus}{\figref{fig:synth-datasets}f}

\CSet{fig:detected-cycles-genus-2-coupled}{\figref{fig:synth-datasets}g}
\CSet{fig:hemb-genus-2-coupled}{\figref{fig:synth-datasets}h}
\CSet{fig:detected-cycles-genus-2}{\figref{fig:synth-datasets}i}
\CSet{fig:hemb-genus-2}{\figref{fig:synth-datasets}j}

\CSet{fig:detected-cycles-genus-4-coupled}{\figref{fig:synth-datasets}k}
\CSet{fig:detected-cycles-genus-4}{\figref{fig:synth-datasets}l}

\CSet{fig:real-datasets}{\ref{fig:real-datasets}}

\CSet{fig:detected-cycles-eth}{\figref{fig:real-datasets}a}
\CSet{fig:detected-cycles-eth-pca}{\figref{fig:real-datasets}b}
\CSet{fig:hemb-eth}{\figref{fig:real-datasets}c}

\CSet{fig:detected-cycles-mda}{\figref{fig:real-datasets}d}
\CSet{fig:detected-cycles-mda-pca}{\figref{fig:real-datasets}e}
\CSet{fig:hemb-mda}{\figref{fig:real-datasets}f}

\CSet{fig:detected-cycles-rna}{\figref{fig:real-datasets}g}
\CSet{fig:detected-cycles-3d-surface}{\figref{fig:real-datasets}h}
\CSet{fig:detected-cycles-islands-aus-nz}{\figref{fig:real-datasets}i}
\CSet{fig:detected-cycles-retina}{\figref{fig:real-datasets}j}

\newcommand{\vd}{{\vec d}}
\newcommand{\ve}{{\vec e}}
\newcommand{\vf}{{\vec f}}

\newcommand{\vv}{{\vec v}}
\newcommand{\vw}{{\vec w}}
\newcommand{\vx}{{\vec x}}
\newcommand{\vy}{{\vec y}}
\newcommand{\vz}{{\vec z}}

\newcommand{\vA}{{\vec A}}
\newcommand{\vB}{{\vec B}}

\newcommand{\vE}{{\vec E}}

\newcommand{\vI}{{\vec I}}

\newcommand{\vL}{{\vec L}}
\newcommand{\vM}{{\vec M}}

\newcommand{\vO}{{\vec O}}

\newcommand{\vW}{{\vec W}}
\newcommand{\vX}{{\vec X}}
\newcommand{\vY}{{\vec Y}}
\newcommand{\vZ}{{\vec Z}}

\newcommand{\cC}{{\mathcal C}}

\newcommand{\cH}{{\mathcal H}}
\newcommand{\cI}{{\mathcal I}}

\newcommand{\cP}{{\mathcal P}}

\newcommand{\cS}{{\mathcal S}}

\newcommand{\fC}{{\mathfrak C}}
\newcommand{\fD}{{\mathfrak D}}

\newcommand{\fN}{{\mathfrak N}}

\renewcommand{\usedbibstyle}{plainnat}

\def\titlename{The decomposition of the higher-order homology embedding constructed from the $k$-Laplacian}
\title{\titlename}

\author{
  Yu-Chia Chen \\
  Electrical \& Computer Engineering \\
  University of Washington \\
  Seattle, WA 98195 \\
  \texttt{yuchaz@uw.edu} \\
  \And
  Marina Meil\u{a} \\
  Department of Statistics \\
  University of Washington \\
  Seattle, WA 98195 \\
  \texttt{mmp2@uw.edu}
}

\begin{document}
\pagenumbering{arabic}

\maketitle

\begin{abstract}
The null space of the $k$-th order Laplacian $\LL_k$, known as the {\em
$k$-th homology vector space}, encodes the non-trivial topology of a
manifold or a network. Understanding the structure of the homology
embedding can thus disclose geometric or topological information from
the data. The study of the null space embedding of the graph Laplacian
$\LL_0$ has spurred new research and applications, such as spectral
clustering algorithms with theoretical guarantees and estimators of
the Stochastic Block Model. In this work, we investigate the geometry
of the $k$-th homology embedding and focus on cases reminiscent of
spectral clustering. Namely, we analyze the {\em connected sum} of
manifolds as a perturbation to the direct sum of their homology
embeddings. We propose an algorithm to factorize the homology
embedding into subspaces corresponding to a manifold's simplest
topological components. The proposed framework is applied to the {\em shortest
homologous loop detection} problem, a problem known to be NP-hard in
general. Our spectral loop detection algorithm scales
better than existing methods and is effective on diverse data such as
point clouds and images.
\end{abstract}

\section{Motivation}
The $k$-th {\em homology vector space} $\cH_k$  provides rich 
geometric information on manifolds/networks. 
For instance, the zeroth, the first, and the second homology
vector spaces identify the connected components, the loops, and the cavities
in the manifold, respectively. 
Topological Data Analysis (TDA)
\cite{WassermanL:18}, as well as other early works in this field, 
aims to extract the dimension of $\cH_k$ and has found wide use 
in  analyzing biological 
\cite{SaggarM.S.G+:18,Kovacev-NikolicV.B.N+:16},
human behavior \cite{AliS.B.S+:07,ZhuX:13}, or other complex systems \cite{WassermanL:18}.
Even though they easily generalize to $k\geq 1$, additional efforts 
are needed to extract topological features (e.g., instances of
loops)
besides ranks due to the combinatorial complexity of the structures 
that support them. 

Spectral methods based on $k$-Laplacians ($\LL_k$), by contrast,
investigate $\cH_k$ in a linear algebraic manner; abundant
geometric information can be extracted from the {\em homology embedding}
$\vY$ (the null space eigenvectors of $\LL_k$) of $\cH_k$.
Analysis of the eigenfunctions (of $\cH_0$) 
\cite{MeilaM.S:01,NgAY.J.W+:02,SchiebingerG.W.Y+:15,CoifmanRR.L:06}
of the graph Laplacian $\LL_0$
is 
pivotal in providing guarantees for spectral clustering 
and community detection algorithms.
Recent advances in this field \cite{ChenYC.M.K+:21,BarbarossaS.S:20a,SchaubMT.B.H+:20}
extend the existing spectral algorithms based on $\LL_0$ to $k \geq 1$;
however, theoretical analysis in $\vY$ of $\cH_k$, unlike spectral clustering, is less developed, in spite of intriguing empirical results by \cite{EbliS.S:19}.
Here, we put these observations on a formal footing based on the
concepts of {\em connected sum} and {\em prime decomposition} of
manifolds (Section \ref{sec:background} and
\ref{sec:prob-formulation}).  We examine these operations through the
lens of the (subspace) perturbation to the homology embedding $\vY$ of
the discrete $k$-Laplacian $\LL_k$ on finite samples (Section
\ref{sec:main-perturbation-theorem}). This framework finds applications
in, i.e., identifying the {\em shortest homologous loops} (Section
\ref{sec:applications}).
Lastly, we
support our theoretical claims with numerous empirical results from
point clouds and images.

\section{Background in Hodge theory and topology}
\label{sec:background}
\paragraph{Simplicial and cubical complex.}
An {\em abstract complex} is a natural extension of a graph 
designed to capture higher-order relationships between its vertices.
A {\em simplicial $k$-complex} (used when the data are point
clouds or networks) is a tuple $\scx_k = (\Sigma_0,\cdots,\Sigma_k)$, 
with $\Sigma_\ell$ being a set of $\ell$ dimensional {\em simplices},
such that every {\em face} of a simplex
$\sigma\in\Sigma_\ell$ is in $\Sigma_{\ell-1}$ for $\ell \leq k$. 
As a side note, a graph $G = (V,E)$ is an $\scx_1$; and $\scx_2 = (V,E,T)$ commonly used in edge 
flow learning
\cite{ChenYC.M.K+:21,SchaubMT.B.H+:20} is obtained by adding a set of
3-cliques (triangles) $T$ of $G$.
This procedure extends to defining
$\Sigma_\ell$ as the set of all $\ell$-cliques of $G$, with the resulting
complex called a {\em clique complex} of the graph $G$.
This complex is also known as a {\em Vietoris-Rips (VR) complex}
if $G$ is the $\epsilon$-radius neighborhood graph used in the
manifold learning literature
\cite{TingD.H.J+:10,CoifmanRR.L:06,ChenYC.M.K+:21},   
The {\em cubical $k$-complex} $\cub_k = (K_0,\cdots,K_k)$ is a complex widely used with image data.  The difference between this complex and the
$\scx_k$ is that a $\cub_k$ is a collection of sets of $\ell$-cubes, for $\ell < k$.
Note that we write $\Sigma_0 = K_0 = V$ the vertex set and
$\Sigma_1 = K_1 = E$ the edge set. $\Sigma_2 = T$ and $K_2 = R$ are the triangle 
and rectangle set, respectively. 
Additionally, we define $n_\ell = |\Sigma_\ell|$ (or $=|K_\ell|$) to
be the cardinality of the $\ell$-dimensional cells and
let $n = n_0$ for simplicity.
For more information about building various complexes
on different datasets please refer to \citet{OtterN.P.T+:17}.

\paragraph{$k$-cochain.}
By choosing an orientation to every $k$-simplex $\sigma_{k,i}\in\Sigma_k$
(or $K_k$),
one can define a finite-dimensional vector space $\mathcal C_k$ ({\em $k$-cochain space}\footnote{We use {\em chain} and {\em cochain} interchangeably for simplicity, see \citet{LimLH:20} for the distinction between them.}). An element $\vec\omega_k = \sum_i\vec\omega_k(\sigma_{k,i})\sigma_{k,i}\in\mathcal C_k$ is called
a {\em $k$-cochain}; one can further express $\vec\omega_k$ as
$\vec{\omega}_k = (\omega_{k,1},\cdots,\omega_{k, n_k})^\top\in\mathbb R^{n_k}$
by identifying each $\sigma_{k,i}$ with the standard basis vector $\ve_i\in\rrr^{n_k}$. Functions on nodes and edge flows, for example,
are elements of $\mathcal C_0$ and $\mathcal C_1$,  respectively.

\paragraph{Boundary matrix.}
The $k$-th {\em boundary matrix} $\vB_k$ \cite{LimLH:20} maps a $k$-cochain of $k$-cells (simplices/cubes)
$\sigma_k$ to the $(k-1)$-cochain of its faces, i.e., 
$\vB_k: \mathcal C_k \to\mathcal C_{k-1}$.
$\vec{B}_k \in\{0,\pm 1\}^{n_{k-1}\times n_k}$ is a sparse binary matrix,
with the sign of the non-zero entries $\sigma_{k-1}, \sigma_k$
given by the orientation of $\sigma_k$ w.r.t. its face $\sigma_{k-1}$.
Hence, different $\scx$ or $\cub$ will induce different $\vB_k$.
For $k=1$ on either the $\scx$ or $\cub$, the boundary map is the
{\em graph incidence matrix}, i.e.,
$(\vB_1)_{[x], [x,y]} = 1$, $(\vB_1)_{[y], [x,y]}=-1$, and zero otherwise;
for $k=2$, each column of $\vec{B}_2$ contains
the orientation of a triangle/rectangle w.r.t. its edges.
Specifically, for an $\scx$,
$(\vB_2)_{[x,y], [x,y,z]} = (\vB_2)_{[x,y], [x,y,z]} = 1$, $(\vB_2)_{[x,z], [x,y,z]} = -1$, and $0$ otherwise;
for a $\cub$,
$(\vB_2)_{[x,y], [x,y,z,w]} = (\vB_2)_{[y,z], [x,y,z,w]} = (\vB_2)_{[z,w], [x,y,z,w]} = 1$, $(\vB_2)_{[x,w], [x,y,z,w]} = -1$, and $0$ otherwise.
Simplex $\sigma_{k+1}$ is a {\em coface} of $\sigma_k$ iff $\sigma_k$
is a face of $\sigma_{k+1}$; let ${\rm coface}(\sigma_k)$ be the set of
all cofaces of $\sigma_k$. The  $(k-1)$-th {\em coboundary matrix} $\vB_k^\top$ (adjoint of $\vB_k$) maps $\sigma_{k-1}$, as a $(k-1)$-cochain, to the $k$-cochain of ${\rm coface}(\sigma_{k-1})$.

\paragraph{$k$-Laplacian.}
Let $\vW_\ell$ be a diagonal non-negative {\em weight matrix}
of dimension $n_\ell$, with $[\vW_\ell]_{\sigma,\sigma}$ representing
the weight of the $\ell$-simplex/cube $\sigma$
and $\vw_\ell \gets \diag(\vW_\ell)$. The weighted 
$k$-Hodge Laplacian \cite{HorakD.J:13} is defined as
\begin{equation}
\label{eq:weighted-laplacian}
\LL_k =
\vA_{k}^\top\vA_{k} + \vA_{k+1}\vA_{k+1}^\top, \,\text{ where }\, \vA_\ell = \vW_{\ell-1}^{-1/2} \vB_\ell \vW_\ell^{1/2} \,\text{ for }\, \ell = k, k+1.
\end{equation}
The weights capture combinatorial or geometric information and must satisfy 
the consistency relation $\vw_\ell(\sigma_\ell)=\sum_{\sigma_{\ell+1}\in {\rm coface}(\sigma_\ell)}\vw_{\ell+1}(\sigma_{\ell+1})$ (in matrix form: $\vw_\ell = |\vB_{\ell+1}|\vw_{\ell+1}$) for $\ell = k, k-1$.
Hence $\vA_k$ can be seen as normalized boundary matrix. 
To determine the weight for the $(k+1)$-simplexes, one can selected
$\vw_{k+1}$ to be constant \cite{SchaubMT.B.H+:20}
or based on (a product of) pairwise distance kernel (for $k=0, 1$)
so that the large sample limit exists \cite{CoifmanRR.L:06,ChenYC.M.K+:21}.
The first and second terms of \eqref{eq:weighted-laplacian} are
called respectively the {\em down} ($\LL_k^{\rm down} = \vA_k^\top\vA_k$)
and {\em up} ($\LL_k^{\rm up} = \vA_{k+1}\vA_{k+1}^\top$) Laplacians.
For $k = 0$, the down component disappears and the resulting
$\LL_0$ is the {\em symmetric normalized} graph Laplacian
used in spectral clustering \cite{vonLuxburgU:07} and Laplacian Eigenmap
\cite{BelkinM.N:03}.

\paragraph{$k$-th homology vector space and embedding.}
The homology vector space $\cH_k$ is a subspace of $\cC_k$ (loop 
space) such that every $k$-cycle (expressed as 
a $k$-cochain) in $\cH_k$ is not the boundary of any $(k+1)$-cochain. 
In mathematical terms, $\cH_k \coloneqq \ker(\vA_k)/ \img(\vA_{k+1})$.
The rank of the subspace is called the $k$-th Betti number $\beta_k=\dim(\cH_k)$, which
counts the number of ``loops'' ({\em homology class}) in the $\scx$.
$\cH_k$ is equivalent to the null space of $\LL_k$ \cite{LimLH:20,SchaubMT.B.H+:20};
therefore, a basis of $\cH_k$ can be obtained by the eigenvectors 
$\vY = [\vy_1,\cdots,\vy_{\beta_1}] \in \rrr^{n_k\times\beta_k}$ of $\LL_k$
with eigenvalue $0$.  
The {\em homology embedding} maps a $k$-simplex $\sigma_k$ to
$\vY_{\sigma_k, :} = [\vy_1(\sigma_k),\cdots\vy_1(\sigma_k)]^\top\in\rrr^{\beta_k}$.
Note that the null space of $\LL_k$ is only identifiable up to a unitary
transformation; hence, the homology embedding might change with a different basis
$\vY$.

\paragraph{Continuous operators on manifolds.}
The $k$-cochains are the discrete analogues of {\em $k$-forms} 
\cite{WhitneyH:05}.
For $k=1$, the  following path integral \cite{WhitneyH:05} 
(along the geodesic $\vec\gamma(t)$ connecting $x$ and $y$) relates a
1-cochain $\vec\omega$ to a 1-form $\vv$ (vector field):
$\vec\omega([x,y]) = \int_x^y \vv(\vec\gamma(t))\vec\gamma'(t) \dd t$.
To estimate a vector field from $\vec\omega$, one can solve a least-squares
problem \cite{ChenYC.M.K+:21}, which is the inverse operation of the
path integral 
(e.g., the vector fields in Figure \ref{fig:genus-2-harmonic-flow} are estimated
from $\vY$). 
Similarly, one can define the {\em differential} $\dd$
and {\em codifferential} $\updelta$ operators 
which are analog to $\vB_{k+1}^\top$
and $\vB_k$, respectively.
The $k$-Laplacian operators, which act on $k$-forms, can be 
defined for manifolds too, i.e., by $\Delta_k = \dd_{k-1}\updelta_k + \updelta_{k+1}\dd_k$.
The homology group (the continuous version of $\cH_k$)
is  defined as the null of $\Delta_k$. Its elements are harmonic $k$-forms $\zeta_k$, computed by solving $\Delta_k\zeta_k =0$ with proper boundary conditions;
they represent the continuous version of the discrete homology basis $\vY$.

\paragraph{Connected sum and manifold (prime) decomposition.}
The connected sum \cite{LeeJM:13} of two $d$ dimensional
manifolds
$\M = \M_1\sharp\M_2$ is built from removing
two $d$ dimensional ``disks'' from each manifold $\M_1$, $\M_2$
and gluing together two manifolds at the boundaries (technical details in \cite{LeeJM:13}). 
The analog of the connected sum for the abstract complexes will be defined
in Section \ref{sec:prob-formulation}.
The connected sum is a core operation
in topology and is related to the concept of manifold (prime)
decomposition. Informally speaking, the prime decomposition
aims to factorize a manifold $\M$ into $\kappa$ smaller building
blocks ($\M = \M_1\sharp\cdots\sharp \M_\kappa$)
so that each $\M_i$ cannot be further expressed as a connected sum of other manifolds.
The well-known {\em classification theorem of surfaces}
\cite{ArmstrongMA:13} states that
any oriented and compact surface is the finite
connected sum of manifolds homeomorphic to either a circle  $\mathbb S^1$,
a sphere $\mathbb S^2$, or a torus $\mathbb T^2$.
Classification theorems for $d>2$
are currently unknown; fortunately, the uniqueness of
the prime decomposition for $d = 3$ was shown
(Kneser-Milnor theorem \cite{MilnorJ:62}).
Recently, \citet{BokorI.C.F+:20} (Corollary 2.5)
showed the existence of factorizations of manifolds with
$d \geq 5$,
even though they might not be unique.

In this paper, we are interested in the following:
given finite
samples from $\M$, which is a $\kappa$-fold connected sum of $\M_i$,
can this decomposition be recovered from the discrete homology embedding $\vY$ 
of $\M$? Namely,
we would like to understand how $\vY$ relates to 
that of each prime manifold $\M_i$.

 \section{Definitions, theoretical/algorithmic aims, and prior works}
\label{sec:prob-formulation}

\paragraph{Definitions.}
The data $\vX$ is sampled from a
$d$-dimensional {\em oriented} manifold $\M$ that can be decomposed into $\kappa$ prime manifolds ($\M = \M_1\sharp\cdots\sharp \M_\kappa$).
Let $\cI_i$ be an index set of the data points in $\vX$ sampled from $\M_i$, for $i=1,\ldots \kappa$.
Denote by $\scx_k$, $\LL_k$, $\cH_k(\M)$, and $\beta_k$ the simplicial complex, 
the $k$-Laplacian, the $k$-homology space, and the $k$-th Betti number of $\M$.
Furthermore, 
let $\hat{\scx}_k^{(i)}=(\hat\Sigma_0^{(i)},\cdots,\hat\Sigma_k^{(i)}),\hat{\LL}_k^{(ii)},\cH_k(\M_i),\beta_k(\M_i)$ be the same quantities for manifold $\M_i$ 
(supported on $\cI_i$ for $i\leq \kappa$).
$\hat{\scx}_k$ and $\hat{\LL}_k$ (without superscript $i$) are the comparable 
notations for the disjoint manifolds $\M_i$'s,
i.e. $\hat{\scx} = \cup_{i=1}^\kappa \hat{\scx}^{(i)} = (\hat\Sigma_0,\cdots,\hat\Sigma_k)$ with 
$\hat\Sigma_\ell = \cup_{i=1}^\kappa \hat\Sigma_\ell^{(i)}$
for $\ell\leq k$,
and $\hat\LL_k$ is a block diagonal matrix with the $i$-th block being $\hat\LL_k^{(ii)}$. 
Additionally, let $\vY$ and $\hat\vY$ (both in $\rrr^{n_k\times \beta_k}$) be the
homology basis of $\LL_k$ and $\hat\LL_k$, respectively. 
Let $\cS_i$ be the index set of columns of $\hat{\vY}$ corresponding to 
homology subspace $\cH_k(\M_i)$, with $\cS_i\cap\cS_j = \emptyset$ for 
$i\neq j$, $|\cS_i| = \beta_k(\M_i)$, and 
$\cS_1\cup\cdots\cup\cS_\kappa = \{1,\cdots,\beta_k\}$. 
Since $\hat\vY$ is the homology embedding of a block diagonal matrix $\hat\LL_k$,
it follows that 
$[\hat\vY]_{\sigma,m}$ equals the homology embedding of $\hat\LL_k^{(ii)}$ 
if $\sigma\in\hat\Sigma_k^{(i)}$ with column $m\in \cS_i$
and is zero otherwise.
Namely, $\hat\vY$ lies in the direct sum of subspaces $\cH_k(\M_i)$ for $i\leq \kappa$.

\paragraph{Theoretical aim.}
We are interested in the geometric properties of 
the null space eigenvectors $\vY$, and specifically in recovering the 
homology basis $\hat\vY$ of the prime manifolds.
Hence, we aim to bound the
{\em distance} between the spaces spanned by $\vY$ and $\hat\vY$.
Under a small perturbation, one can provide an analogous argument to the
{\em orthogonal cone} structure \cite{NgAY.J.W+:02,SchiebingerG.W.Y+:15} in
spectral clustering (the zeroth homology embedding).
The main technical challenge is that the connected sum of manifolds is a 
highly localized perturbation; namely, most cells are not affected at all, while
those involved in the gluing process gain or lose $\mathcal O(1)$ (co)faces.
Without properly designing $\LL_k$ and $\hat\LL_k$,
one might get a trivial bound.

\begin{wrapfigure}[10]{r}{0.523\textwidth}
\centering
    \vspace{-20pt}
    \includegraphics[width=.99\linewidth]{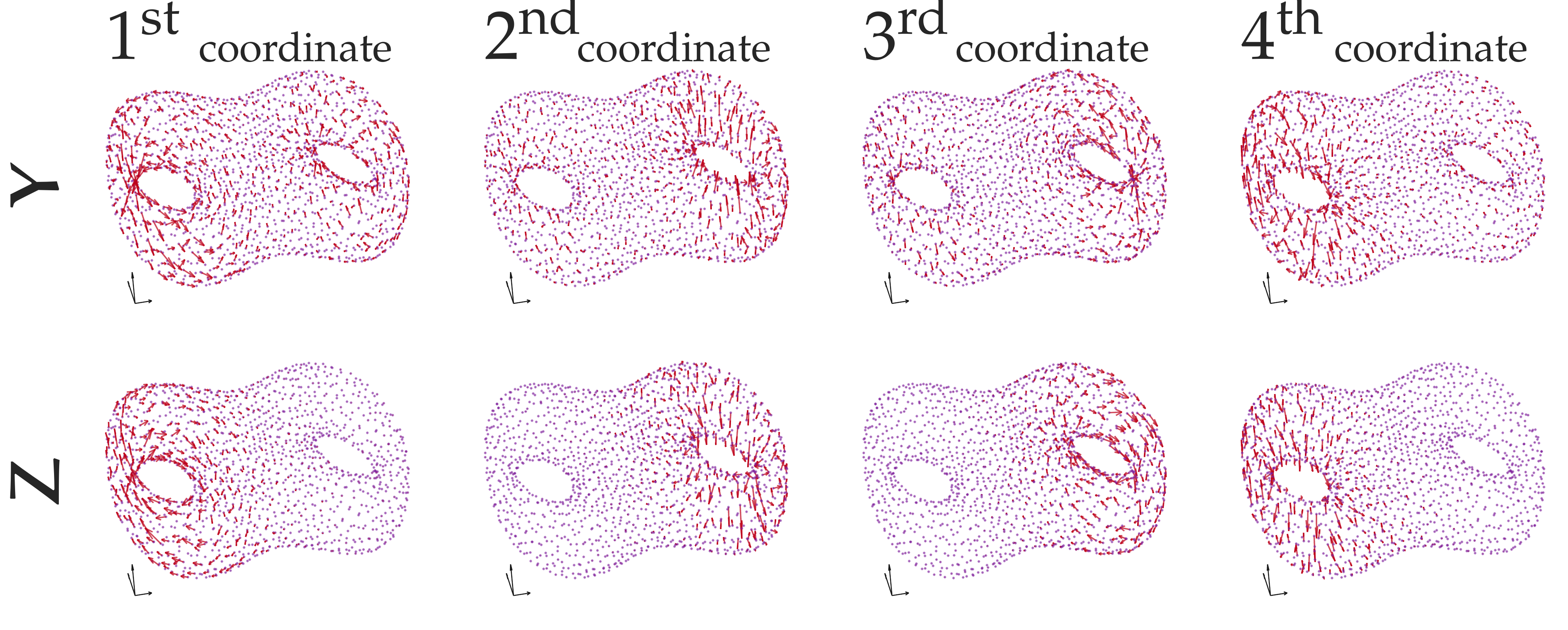}
    \caption{Harmonic vector fields obtained by solving a least-squares \cite{ChenYC.M.K+:21} with $\vY$ (top) and $\vZ$ (bottom).}
    \label{fig:genus-2-harmonic-flow}
\end{wrapfigure}
\paragraph{Algorithmic aim.} 
We exemplify the algorithmic aim using $k = 1$, $d = 2$, and $\kappa = 2$,
particularly the genus-2 surface shown in Figure \ref{fig:genus-2-harmonic-flow}.
The null space basis $\vY$ of $\LL_k$ is only identifiable up to a unitary matrix
due to the multiplicity of the zero eigenvalues. For instance, 
the top and bottom rows of Figure \ref{fig:genus-2-harmonic-flow} are both
valid bases for the edge flow in $\cH_1$.
However, the basis vector fields in the second row of Figure 
\ref{fig:genus-2-harmonic-flow} are more interpretable than those in the top row
because $\vY$ (the first row) is a linear combination
of $\vZ$ (the second row), with each basis (column in the figure) corresponding to 
a single homology class (loop).
Therefore, here we propose a {\em data-driven} approach to obtain the optimal
basis $\vZ$
such that the coupling from other manifolds/subspaces is as weak as possible.
Being able to obtain $\vZ$ from an arbitrary $\vY$ can support numerous
applications (more in Section \ref{sec:applications});
however, 
it is difficult to design a criterion for finding the optimal
$\vZ$ without knowing the geometric structure of $\cH_k$.

\paragraph{Prior works.}
The shape of the embedding of the graph Laplacian $\LL_0$ is pivotal
for showing the guarantees of spectral clustering
algorithms for point cloud data or the inference algorithms for the stochastic block model.
The analyses used either the
matrix perturbation theory \cite{NgAY.J.W+:02,WanY.M:15,vonLuxburgU:07}
or assume a mixture model \cite{SchiebingerG.W.Y+:15}.
For the higher-order $k$-Laplacian, it is reported
empirically that the homology embedding is
approximately distributed on the union (directed sum) of subspaces 
\cite{EbliS.S:19};
subspace clustering algorithms \cite{KailingK.K.K+:04}
were applied to partition edges/triangles under their framework.

 \section{Main result: connected sum as a matrix perturbation}
\label{sec:main-perturbation-theorem}

In this section, we analyze the geometric structure of $\vY$ by viewing the operation of {\em connected sum}
through the lens of matrix perturbation theory \cite{StewartGW.S.S+:90}. We show
that, under certain conditions, the homology embedding $\vY$ of the joint Laplacian  
$\LL_k$ is approximated by $\hat\vY$ for the simplexes that are not created/destroyed during connected sum.
In matrix terms, we show that
$\vY \approx \hat\vY\vO$ (Theorem \ref{thm:subspace-perturb-bound-Lk}) with $\vO$ a unitary transformation.

We first prepare our 
assumptions suited for $\scx$ built from point clouds. Most of the assumptions (except Assumption \ref{assu:basic-ml-assu}
for which the connected sum might not be defined)
can be extended to the clique complex (for networks) or cubical complex
(for images) without too many modifications.
\begin{assumption}
	The point cloud $\vX\in \rrr^{n\times D}$ is sampled from a
        $d$-dimensional oriented and compact manifold $\M \subseteq
        \rrr^{n\times D}$; the homology vector spaces $\mathcal
        H_k(\scx)$ formed by the simplicial complex constructed
        from $\vX$ are isomorphic to the homology group
        $\mathcal H_k(\M)$ of $\M$, i.e., $\mathcal H_k(\scx)
        \simeq \mathcal H_k(\M)$. 
Furthermore, assume that $\M = \M_1\sharp \cdots\sharp \M_\kappa$, and that $\mathcal H_k(\hat\scx^{(i)})
        \simeq \mathcal H_k(\M_i)$ for $i = 1,\cdots, \kappa$. 
\label{assu:basic-ml-assu}
\end{assumption}
This assumption is the minimal assumption needed for the analysis of the embedding
of the $\LL_k$;
it states that any procedure to construct the simplicial complex or weight
function for $\LL_k$ is accepted
as long as the isomorphic condition holds. The construction of the $\scx$
from the point cloud is out of the scope of this manuscript (see, e.g., \citet{ChenYC.M.K+:21} for building $\LL_1$ from $\vX$ with an analyzable limit). The last condition
requires that the manifold $\M$ can be decomposed; this is most likely true,
except for the known hard case of $\M$ with $d = 4$ discussed in Section \ref{sec:background}.
To make this assumption hold for networks or images,
one can require that $\LL_k$ constructed from these two datasets can be 
roughly factorized into block-diagonal entries.
Below we provide two other assumptions that are valid for both $\scx$ and $\cub$ (with some modifications): 
the first one controls the eigengap and the second 
one ensures a small perturbation in the spectral norm of $\LL_k - \hat\LL_k$. 
By construction, $\LL_k$ is positive semi-definite; since we are interested 
in the stability of its null space, we define, for any matrix $\mathbf L\succeq 0$, the {\em eigengap} as the the smallest {\em non-zero} eigenvalue of $\mathbf L$ and denote it $\lambda_{\rm min}(\mathbf L)$.
\begin{assumption}
	We denote the set of destroyed and created $k$-simplexes during connected sum by
	$\fD_k$ and $\fC_k$, respectively; let the set of non-intersecting 
	simplexes be $\fN_k = \Sigma_k \backslash \fC_k = \hat\Sigma_k\backslash\fD_k$.
	We have: (1) no $k$-homology class is created
	during the connected sum process, i.e.,
	$\beta_k(\scx) = \sum_{i=1}^\kappa \beta_k(\hat\scx^{(i)})$.
	(2) The eigengaps of $\LL_k^{\fC,\fC}$ and $\hat\LL_k^{\fD,\fD}$ are bounded away from the eigengaps of $\LL_k^{(ii)}$, i.e., $\min\{\lambda_{\min}(\LL_k^{\fC,\fC}), \lambda_{\min}(\hat\LL_k^{\fD,\fD})\} \gg \min\{ \delta_1,\cdots, \delta_\kappa \}$,
	where $\delta_i$ is the {\em eigengap} of $\LL_k^{(ii)}$.
\label{assu:connected-sum-preserve-homology}
\end{assumption}
The first condition requires that the intersecting simplexes $\fD_k\cup\fC_k$
do not create or destroy any $k$-th homology class; this holds, for instance, when the manifold $\M$ has dimension $d > k$. 
Under this condition, we have
$\mathcal H_k(\M_1\sharp\M_2) \simeq \mathcal H_k(\M_1) \oplus \mathcal H_k(\M_2)$ 
\cite{LeeJM:13}. 
A counterexample for this condition is, e.g., inspecting the cavity 
space ($k=2$) of a genus-2 surface built from gluing two tori together. That is,
$\beta_2$ of a genus-2 surface is 1, while the sum of $\beta_2$ of two tori is 2.
The second condition requires that the principal
submatrix of $\LL_k$ described by the block of $\fC_k\cup\fD_k$ has large eigengap. This happens, e.g., when $\fC_k$ and $\fD_k$ are cliques
and are contained in small balls.

\begin{assumption}[Informal]
	Let $\tilde\vw_k = |\vB_{k+1}[\fN_k, \fN_{k+1}]| \vw_{k+1}$, $\tilde\vw_{k-1} = |\vB_{k}[:, \fN_{k}]|\tilde\vw_{k}$. 
	For $\ell = k$ or $k-1$, we have $\max_{\sigma\in\fN_\ell} \left\{w_\ell(\sigma) / \tilde w_\ell(\sigma) - 1\right\} \leq \epsilon_\ell$,
	$\max_{\sigma\in\fN_\ell} \left\{\hat w_\ell(\sigma) / \tilde w_\ell(\sigma) - 1\right\} \leq \epsilon_\ell$, and
	$\max_{\sigma\in\fN_\ell} \left\{|w_\ell(\sigma) / \hat w_\ell(\sigma) - 1|\right\} \leq \epsilon'_\ell$.
Assumption \ref{assu:degree-of-simplex-increase-less-formal} 
	is the formal version of this assumption.
\label{assu:degree-of-simplex-increase-less}
\end{assumption}

For $k=1$, it states that not too many triangles are being
created or destroyed during connected sum. 
For this assumption to hold, the density in the connected sum region should 
be smaller than in other regions, i.e., the manifold $\M$ should be sparsely connected (e.g., Figure \figref{fig:hemb-2-loops}).
Empirically, we observed that the perturbation is small even when $\M$
is not sparsely connected (more discussions in Section \ref{sec:experiments}).
Note also that $\epsilon'_\ell \ll \epsilon_\ell$, for $\epsilon'_\ell$
represents the {\em net} change in the degree after connected sum.
It might be possible to obtain a tighter bound fully by $\epsilon'_\ell$'s,
which do not depend on the relative density between the connected sum region
and the remaining manifolds; we leave it as future work.

\begin{theorem}
Let $\difflap_k^{\rm down}$ be the modified difference (defined in \suppname{} \ref{sec:proof-subspace-perturbation})
of $\LL_k^{\rm down}$ and $\hat{\LL}_k^{\rm down}$, same for that ($\difflap_k^{\rm up}$)
of {\em up} Laplacians.
Under Assumptions \ref{assu:basic-ml-assu}--\ref{assu:degree-of-simplex-increase-less} with notations defined as before and $\lambda_k = k+2$, if $\left\|\difflap_k^{\rm down}\right\|^2 \leq \left[2\sqrt{\epsilon_k'} + \epsilon_k' + \left(1+\sqrt{\epsilon_k'}\right)^2 \sqrt{\epsilon_{k-1}'} + 4\sqrt{\epsilon_{k-1}}\right]^2\lambda^2_{k-1}$
and $\left\| \difflap_k^{\rm up}\right\|^2 \leq \left[2\sqrt{\epsilon_k'} + \epsilon_k' + 2\epsilon_k + 4\sqrt{\epsilon_{k}}\right]^2\lambda^2_{k}$, then there exists a unitary matrix $\vO\in\rrr^{\beta_k\times\beta_k}$ such that
\begin{equation}
	\left\| \vY_{\fN_k,:} - \hat\vY_{\fN_k,:}\vO \right\|_F^2 \leq \frac{8\beta_k\left[\left\| \difflap_k^{\rm down} \right\|^2+\left\| \difflap_k^{\rm up} \right\|^2\right]}{\min\{\delta_1,\cdots,\delta_\kappa\}}.
\label{eq:main-Lk-perturbation-bound}
\end{equation}
\label{thm:subspace-perturb-bound-Lk}
\end{theorem}

\begin{wrapfigure}[13]{r}{0.48\textwidth}
\IncMargin{1.6em}
\vspace{-18pt}
\begin{algorithm}[H]
    \setstretch{1.15}
    \SetKwInOut{Input}{Input}
    \SetKwInOut{Output}{Return}
    \SetKwComment{Comment}{$\triangleright $\ }{}
    \Indm
	\Input{$\scx$, $k$, weights $\vW_{k+1}$}
	\Indp
    $\vB_k, \vB_{k+1} = \textsc{BoundaryMaps}(\scx, k)$ \label{algstp:bmap} \, \Comment{in Algorithm \ref{alg:boundary-map}}
    \For{$\ell = k, k-1$}{
	$\vec{W}_\ell \gets \diag\{|\vec{B}_{\ell+1}|\vec{W}_{\ell+1}\vec{1}_{n_{\ell+1}}\}$ \\
$\vA_{\ell+1} \gets \vW_\ell^{-1/2}\vB_{\ell+1}\vW_{\ell+1}^{1/2}$
}
	$\LL_k = \vA_k^\top\vA_k + \vA_{k+1}\vA_{k+1}^\top$ \\
	$\vY \in\rrr^{n_k\times \beta_k} \gets \textsc{NullSpace}(\LL_k)$ \\
	$\vZ \gets \textsc{ICANoPrewhite}(\vY)$ \label{algstp:noprewhiteica} \\
	\Indm
	\Output{Independent basis $\vZ$}
	\Indp
\caption{Subspace identification}
\label{alg:ica-no-prewhite}
\end{algorithm}
\DecMargin{1.6em}
\end{wrapfigure}
The proof (in \suppname{} \ref{sec:proof-subspace-perturbation}) is based 
on bounding the error between $\LL_k$ and $\hat\LL_k$ with $\tilde\LL_k$ 
(the Laplacian after removal of $k$-simplices during connected sum), the 
use of a variant of the Davis-Kahan theorem \cite{YuY.W.S+:15},
and the bound of the spectral norm of $\LL_k$ for a simplicial complex, i.e., $\|\LL_k\|_2\leq \lambda_k = k+2$ 
\cite{HorakD.J:13}.

What is unusual for the bound is that the LHS of 
\eqref{eq:main-Lk-perturbation-bound} contains only the simplices
in $\fN_k$. It is unlikely that one can get a small bound for the 
simplices in $\fC_k\cup\fD_k$ since they do not exist before or after 
gluing manifolds together. Nonetheless, \eqref{eq:main-Lk-perturbation-bound}
makes sure that the (unbounded) perturbations in the embedding of 
$\fC_k\cup\fD_k$ do not propagate to the rest of the simplices.
The bound in \eqref{eq:main-Lk-perturbation-bound} can be extended to 
$\cub$ (Corollary \ref{thm:subspace-perturb-bound-Lk-cub})
by changing the $\lambda_k$ value from $(k+2)$ to $2k+2$. The $2k+2$ term here is 
the maximum eigenvalue of the $\LL_k$ built from any cubical complex 
(Proposition \ref{thm:spectral-norm-Lk-cb}).

\begin{corollary}[For $\LL_k$ built from a $\cub$]
	Under Assumptions \ref{assu:connected-sum-preserve-homology}--\ref{assu:degree-of-simplex-increase-less} with $\difflap_k^{\rm up}$ as well as $\difflap_k^{\rm down}$ defined in Theorem \ref{thm:subspace-perturb-bound-Lk} and $\lambda_k = 2k+2$, there exists a unitary
	matrix $\vO$ such that \eqref{eq:main-Lk-perturbation-bound} holds.
\label{thm:subspace-perturb-bound-Lk-cub}
\end{corollary}

\paragraph{Subspace identification.}
We propose to (approximately) separate the columns of the coupled basis
$\vY$ to an independent basis $\vZ$ (as an approximation to $\hat\vY$), with columns being a permutation of
$\{1,\cdots,\beta_k\}$, by {\em blind source separation}, as described by  
Algorithm \ref{alg:ica-no-prewhite}. 
Specifically, $\vZ$ is obtained by Infomax ICA  
\cite{BellAJ.S:95} on $\vY$ of $\LL_k$, with a modification (Line
\ref{algstp:noprewhiteica}) that preserves
the necessary properties of harmonic cochains (i.e., they are
{\em divergence-free} and {\em curl-free}, see also Proposition \ref{thm:digraph-reachable-no-shortcuit}).
Algorithm \ref{alg:ica-no-prewhite} works for $\cub$ as well by using
the appropriate $\vB_k$, $\vB_{k+1}$ construction method (Line
\ref{algstp:bmap}).

\section{Applications: homologous loops detection, clustering, and visualization}
\label{sec:applications}

\begin{wrapfigure}[21]{r}{0.62\textwidth}
\vspace{-15pt}
\IncMargin{1.6em}
\begin{algorithm}[H]
    \setstretch{1.15}
    \SetKwInOut{Input}{Input}
    \SetKwInOut{Output}{Return}
    \SetKwComment{Comment}{$\triangleright $\ }{}
    \Indm
	\Input{$\vZ = [\vz_1,\cdots,\vz_{\beta_1}]$, $V$, $E$, edge distance $\vd$}
	\Indp
	\For{$i = 1,\cdots,\beta_1$}{
		$E_i^+ \gets \{ (s,t): (s,t)\in E \text{ and } [\vz_i]_{(s,t)} > 0\}$ \label{algstp:eiplus} \\
		$E_i^- \gets \{ (t,s): (s,t)\in E \text{ and } [\vz_i]_{(s,t)} < 0 \}$ \label{algstp:eiminus} \\
		$\tau\gets \textsc{Percentile}(|\vz_i|, 1-1/\beta_1)$ \label{algstp:thres-val} \\
		$E_i^\times \gets \{e \in E_i^+ \cup E_i^-: |[\vz_i]_e| < \tau\}$  \label{algstp:remove-edges-smaller-thres}\\
		$E_i \gets E_i^+ \cup E_i^- \backslash E_i^\times $\\
		$G_i \gets (V, E_i)$, with weight of $e\in E_i$ being $[\vd]_e$ \label{algstp:weighted-Gi} \\
		$d_{\min} = \texttt{inf}$ \\
		\For{$e=(t, s_0) \in E_i$ \label{algstp:for-all-edges} }{
			$\cP^*, d^* \gets \textsc{Dijkstra}(G_i,\texttt{from=}s_0, \texttt{to=}t)$\,
			\Comment{Note that $\cP^* = [s_0,s_1,\cdots,t]$}
			\If{$d^* < d_{\min}$}{
				$\cC_i \gets [t, s_0,s_1,\cdots, t]$ \label{algstp:assign-ci}
			}
		}
	}
	\Indm
	\Output{ $\cC_1,\cdots,\cC_{\beta_1}$}
	\Indp
	\caption{Spectral homologous loop detection}
    \label{alg:dijkstra-shortest-loop}
\end{algorithm}
\DecMargin{1.6em}
\end{wrapfigure}
\paragraph{Homologous loop detection.}
In addition to the {\em rank} information available from classical TDA methods, 
one might find it
beneficial to extract the shortest cycle of the corresponding $\cH_k$ 
generator.
This application is found useful in domains including finding
minimum energy trajectories in molecular dynamics datasets,
trajectory inference in RNA single-cell sequencing \cite{SaelensW.C.T+:19},
and segmenting circular structures in medical images
 \cite{SinghN.C.M+:14}.
We propose a {\em spectral} shortest homologous
 loop detection algorithm
(Algorithm \ref{alg:dijkstra-shortest-loop}) based on the shortest path
algorithm (Dijkstra) as follows: for each dimension $i = 1,\cdots,\beta_1$,
the algorithm reverses every edge $e$ having negative $[\vz_i]_e$ to generate
a weighted digraph $G_i = (V, E_i)$ (Lines \ref{algstp:eiplus}--\ref{algstp:ei}),
with the weight of edge $e\in E_i$ equal to the Euclidean 
distance $[\vd]_{(i,j)} = \|\vx_i-\vx_j\|_2$. The algorithm finds a 
shortest (in terms of $\vd$) loop on this weighted digraph for each $i$
and outputs it as the homologous loop representing the $i$-th class. 
We present the following proposition (with the proof in \suppname{} \ref{sec:proofs-of-applications-props}) to support Algorithm \ref{alg:dijkstra-shortest-loop}; 
it implies that if each coordinate of $\vZ$ extracted from Algorithm
\ref{alg:ica-no-prewhite} corresponds to a homology class, then the detected
homologous loop for each homology class is the shortest.

\begin{proposition}
	Let $\vz_i$ for $i = 1,\cdots,\beta_1$ be the $i$-th homology basis
	that corresponds to the $i$-th homology class.
	For every $i = 1,\cdots,\beta_1$, (1) there exist at least one {\em cycle} in 
	the digraph $G_i$ such that every vertex $v\in V$ can traverse back to itself ({\em reachable}); 
	(2) the corresponding {\em cycle} will enclose at least one homology class ({\em no short-circuiting}).
\label{thm:digraph-reachable-no-shortcuit}
\end{proposition}

Since every vertex is {\em reachable} from itself, we are guaranteed to find a
loop for any starting/ending pair (Lines \ref{algstp:for-all-edges}--\ref{algstp:assign-ci}). 
Additionally, there will be no short-circuiting for any loop; 
each loop we found from Dijkstra is guaranteed to be non-trivial.
However, there is one caveat from the second property:
even though the $i$-th loop is non-trivial, it might not always be 
corresponding to the $i$-th homology class due to the noise in small
$[\vz_i]_e$. Namely, loops that do not represent $i$-th homology class
can be formed with edges $e$ having small $[\vz_i]_e$,
resulting in the instability and the (possible) duplication of the 
identified loops. To address the issue, we propose a heuristic thresholding, 
by which we keep the $n_1/\beta_1$ edges with the largest absolute
value in $|[\vz_i]_e|$ (Lines \ref{algstp:thres-val}--\ref{algstp:remove-edges-smaller-thres}). 
We chose to keep $n_1/\beta_1$ by treating each homology class equally, i.e.,
each class has roughly $n_1/\beta_1$ edges.

Compared with previous approaches that find the shortest loops \cite{DeyTK.S.W+:10} combinatorially, our approach has better time complexity; specifically, the algorithm of \cite{DeyTK.S.W+:10} has time complexity
$\mathcal O(nn_1^3 + nn_1^2n_2)$, whereas Algorithm
\ref{alg:dijkstra-shortest-loop} runs in time $\mathcal
O(n_1^{2.37\cdots} + \beta_1^2n_1 + \beta_1n_1n\log n)$. The
first, second, and  third terms correspond to the time
complexity of eigendecomposition of $\LL_1$, the Infomax ICA, and the
Dijkstra algorithm on every digraph $G_i$,
respectively.  Note that if the simplicial complex is built from point
clouds, the number of triangles $n_2$ may be large; this
dependency on $n_2$  makes the algorithm
\cite{DeyTK.S.W+:10} hard to scale.
On the other hand, our framework requires that $\vz_i$ are each supported on one
homology class; therefore, 
loops can only be correctly identified using Algorithm \ref{alg:dijkstra-shortest-loop} if the manifold is sparsely connected (Assumptions \ref{assu:basic-ml-assu}--\ref{assu:degree-of-simplex-increase-less}).

\paragraph{Classifying any 2-dimensional manifold.}
The Betti number $\beta_1$ of a torus is 2, which is equal to that of two 
disjoint circles; hence one cannot distinguish these two manifolds 
{\em only} by rank information.
Fortunately, they can be categorized using the homology embedding $\vZ$.
By the classification theorem
\cite{ArmstrongMA:13}, any 2D surface is the connected sum of circles $\mathbb S^1$ and tori $\mathbb T^1$; therefore, Theorem
\ref{thm:subspace-perturb-bound-Lk} indicates that embedding lies approximately 
in the directed sum of homology subspace of $\mathbb S^1$
and/or $\mathbb T^2$.
The homology embedding of $\mathbb S^1$ is a line since it is in $\rrr^1$.
On the other hand,
any loop in a torus can be a convex combination of the two homology classes,
implying that the intrinsic dimension of the homology embedding is 2. 
It is hard to obtain $\vZ$ of any arbitrary torus; we present the homology embedding of the flat
$m$-torus below by expressing the null space basis (1-cochains) as the path 
integrals of the corresponding harmonic
1-forms \cite{WhitneyH:05,ChenYC.M.K+:21}.
\begin{proposition}
	The envelope of the first homology embedding (1-cochain) 
	induced by the harmonic 1-form on the flat m-torus $\mathbb T^m$ is an 
	$m$-dimensional ellipsoid.
\label{thm:harmonic-emb-of-Tk}
\end{proposition}

The proof (in \suppname{} \ref{sec:proofs-of-applications-props}) is straightforward thus is omitted here.
Proposition \ref{thm:harmonic-emb-of-Tk} and the classification theorem
suggest that the first homology embedding is either a line, a disk, or
a combination of the two (with replacement). See an example for the genus-2 surface in 
Figures \figref{fig:hemb-genus-2} and \ref{fig:pairplot-gen2}.

\paragraph{Other applications.}
As pointed out earlier, one can visualize the basis of the harmonic vector fields
(of $\cH_k$) by overlaying the columns of $\vY$ onto the original dataset (Figure \ref{fig:genus-2-harmonic-flow}). 
Being able to successfully extract a decoupled basis $\vZ$ increases the 
interpretability of $\cH_k$, as shown in the second row of Figure \ref{fig:genus-2-harmonic-flow}.
Theorem \ref{thm:subspace-perturb-bound-Lk}
also  
supports the use of subspace clustering algorithm 
in the higher-order simplex clustering framework \cite{EbliS.S:19}.

 \section{Experiments}
\label{sec:experiments}
We demonstrate our approach by computing $\vY,\vZ$ and the shortest
loops for five synthetic manifolds: two of them are prime manifolds
(\dattorus{} {\em torus}, \datthreetorus{} {\em three-torus}) and
three (\dattwoholes{} {\em punctured plane with two holes},
\datgenustwo{} {\em genus-2}, and \datgenusfour{} {\em concatenation of 4 tori}) are factorizable manifolds.  Furthermore, five additional
real point clouds (\datethanol{} and \datmda{} from chemistry,
\datrna{} from biology, \datbudd{} from 3D modeling, and \datisland{}
from oceanography) are analyzed under this framework. For all the
point clouds, we build the VR complex $\scx$ from the CkNN kernel
\cite{BerryT.S:19} so that the resulting $\LL_1$ is sparse and the
topological information is preserved.
Note that other methods for building an $\scx$ from $\vX$ can also be used
as long as $\cH_k$ is successfully identified (Assumption
\ref{assu:basic-ml-assu}).
Lastly, we illustrate the
efficacy of our framework to a non-manifold data: \datretina{}
from medical imaging.
Please refer to \suppname{} \ref{sec:supp-datasets-detail} for
detailed discussions on procedures to generate, preprocess, and
download these datasets.  All experiments are replicated more than
five times with similar results.  We perform our analysis on a desktop
running Linux with 32GB RAM and an 8-Core 4.20GHz
Intel\textregistered~Core\texttrademark~i7-7700K CPU; every experiment
completes within 3 minutes (1-2 minutes on eigendecomposition of
$\LL_1$, and around 30 seconds on both ICA and Algorithm
\ref{alg:dijkstra-shortest-loop}).

\paragraph{Synthetic manifolds.} The results for the synthetic manifolds
are in Figure \ref{fig:synth-datasets}. Figure \figref{fig:hemb-2-loops}
(the harmonic embedding of \dattwoholes{}) confirms
Theorem \ref{thm:subspace-perturb-bound-Lk} that
$\vY$ is approximately distributed on two subspaces (yellow and red), with each 
loop parametrizing a single hole (inset of Figure \figref{fig:hemb-2-loops}).
As discussed previously in Figure \ref{fig:genus-2-harmonic-flow}, the harmonic 
vector bases (green and blue) are mixtures of the separate subspaces; therefore, these bases
have poor interpretability compared with the independent subspace $\vZ$
identified by Algorithm \ref{alg:ica-no-prewhite}. 
The shortest loops
(Figure \figref{fig:detected-cycles-2-loops}) corresponding to $\vz_1$ (yellow),
$\vz_2$ (red) are obtained by running Dijkstra on the digraphs induced by
$\vz_1$ and $\vz_2$ separately (Algorithm \ref{alg:dijkstra-shortest-loop}).
Figures \figref{fig:detected-cycles-2-torus}--\figref{fig:hemb-3-torus}
show the results of the two simple {\em prime manifolds}:
\dattorus{} and \datthreetorus{}. The harmonic embeddings of \dattorus{} (Figure
\figref{fig:hemb-2-torus}) and \datthreetorus{} (Figure \figref{fig:hemb-3-torus})
are a two-dimensional disk and a three-dimensional ellipsoid, respectively;
this confirms the conclusion from Proposition \ref{thm:harmonic-emb-of-Tk}.
The shortest loops obtained from Algorithm \ref{alg:dijkstra-shortest-loop}
for these two datasets are in Figures \figref{fig:detected-cycles-2-torus}
and \figref{fig:detected-cycles-3-torus}, showing that these loops
travel around the holes in \dattorus{} (or \datthreetorus{}). Note that
we plot \datthreetorus{} in the intrinsic coordinate because
a three torus can not be embedded in 3D without breaking neighborhood 
relationships. Three lines in
\figref{fig:detected-cycles-3-torus} are indeed loops due to the periodic
boundary condition, i.e., $0 = 2\pi$, in the intrinsic coordinate. 
Figures \figref{fig:hemb-genus-2-coupled} and \figref{fig:hemb-genus-2}
show the embedding of the coupled harmonic basis ($\vY$) and that
corresponding to the independent subspace ($\vZ$) obtained by Algorithm \ref{alg:ica-no-prewhite}.
Compared with $\vY$, each coordinate of $\vZ$ corresponds to a subspace, i.e.,
the left or right handle of \datgenustwo, and does not couple with other homology
generators. $\vZ$ is thus a union of two 2D disks,
with each disk approximating the harmonic embedding of a torus 
(see Figure \ref{fig:pairplot-gen2} for more detail).
Compared with the loops obtained by running Algorithm 
\ref{alg:dijkstra-shortest-loop} on $\vY$ (Figure \figref{fig:detected-cycles-genus-2-coupled}), 
each loop in Figure \figref{fig:detected-cycles-genus-2} identified from $\vZ$  
parameterizes the corresponding homology generator without being homologous to 
other loops. Similar results on \datgenusfour{} are in Figures
\figref{fig:detected-cycles-genus-4-coupled} and \figref{fig:detected-cycles-genus-4},
which correspond to the loops obtained from $\vY$ and $\vZ$, respectively.
The pairwise scatter plots of the eight-dimensional $\vZ$ (or $\vY$)
are in Figure \ref{fig:pairplot-gen4} of \suppname{} \ref{sec:supp-datasets-detail}.
Note that \dattwoholes{} is an example of a sparsely connected manifold
(see the low-density area in the middle), with $\epsilon_1 \approx 0.035$
and $\epsilon_0 \approx 0.038$. Manifolds of other synthetic/real datasets 
might not be sparsely connected due to the (approximately) constant
sampling densities; nevertheless, the perturbations to the subspaces
remain small for these datasets.
\begin{figure}[t]
    \centering
    \includegraphics[width=0.99\textwidth]{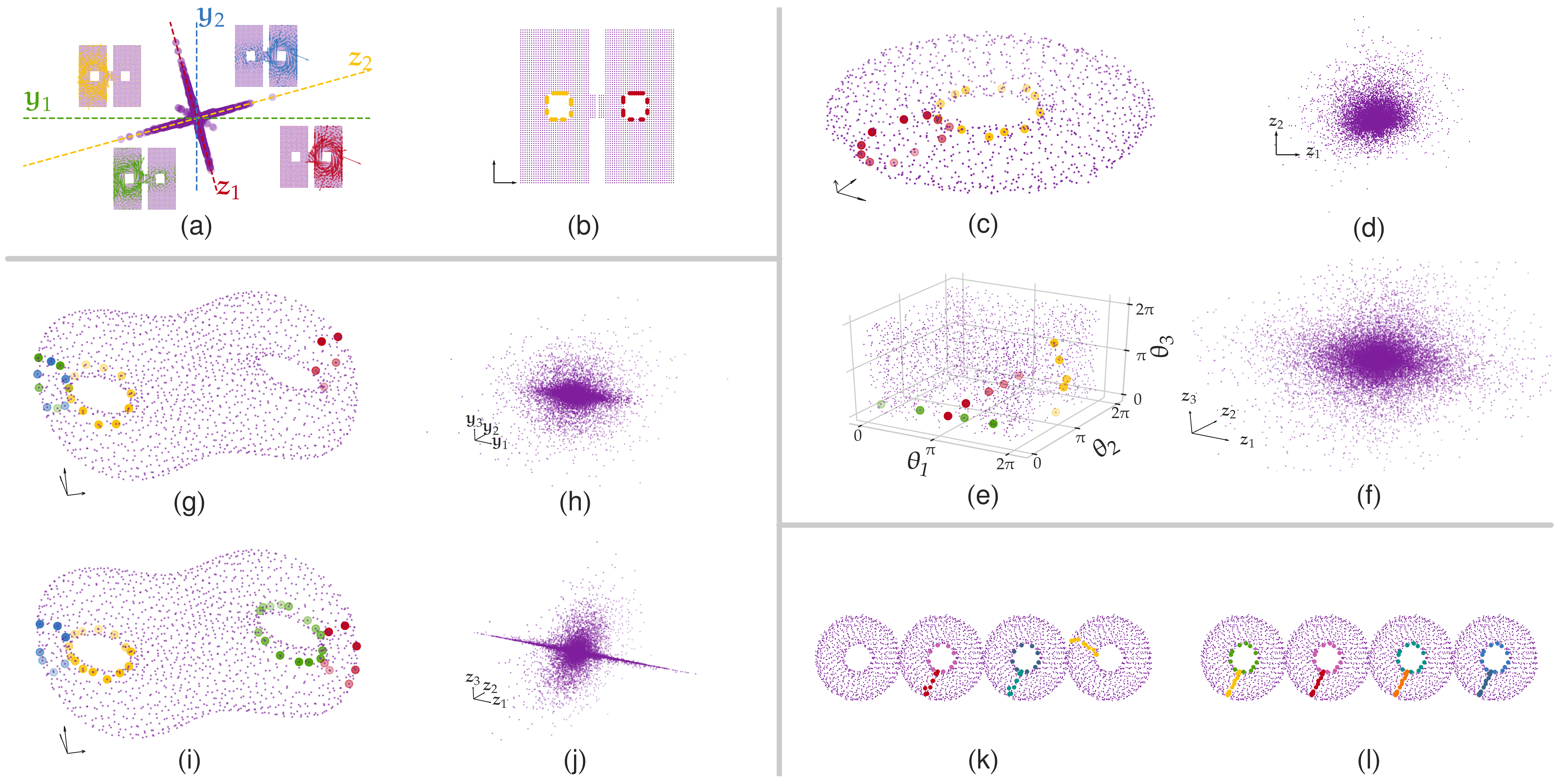}
    \caption{(a) The first homology embedding of \dattwoholes. The harmonic vector fields are overlaid on the data in the inset plots; green,  blue, red, and yellow arrows correspond to $\vy_1$, $\vy_2$, $\vz_1$, and $\vz_2$, respectively. (b), (c), (e), (i), and (l) are the detected loops using Dijkstra on $\vZ$ for \dattwoholes{} (colors are in (a)), \dattorus, \datthreetorus, \datgenustwo, and \datgenusfour, respectively. (g) and (k) represent the identified loops on the coupled embedding $\vY$ for \datgenustwo{} and \datgenusfour, respectively. (d), (f), (h), and (j) present the embeddings used to detect loops in (c), (e), (g), and (i), respectively.}
    \label{fig:synth-datasets}
\end{figure}

\paragraph{Small molecule data \cite{ChmielaS.T.S+:17}.}
Figures \figref{fig:detected-cycles-eth}--\figref{fig:hemb-eth}
and \figref{fig:detected-cycles-mda}--\figref{fig:hemb-mda} show our
analysis on \datethanol{} and \datmda{}, respectively. These two small
molecule datasets, whose ambient dimensions are $D=102$ and $D=98$,
are suggested to be noisy non-uniformly sampled tori \cite{TaborM:89};
the harmonic embeddings of these two datasets (Figures \figref{fig:hemb-eth}
and \figref{fig:hemb-mda} ) confirm this idea.
Finding the minimum trajectories corresponding to a specific bond torsion
is of interest in chemistry; in these two molecular dynamics systems, this problem
can be translated into finding the homologous loops in the point cloud.
The homologous loops found
by Algorithm \ref{alg:dijkstra-shortest-loop} overlaid on the first three
{\em principal components} (PCs) for these two datasets can be found in
Figures \figref{fig:detected-cycles-eth-pca} (for \datethanol{}) and 
\figref{fig:detected-cycles-mda-pca} (for \datmda{}). The identical homologous
loops
plot in the bond torsion space (with definition in the insets) based on our prior 
knowledge are in Figures \figref{fig:detected-cycles-eth}
and \figref{fig:detected-cycles-mda}. Similar to the discussion for \datthreetorus{}
(Figure \figref{fig:detected-cycles-3-torus}), the yellow/red trajectories
form loops due to the periodic boundary condition of the bond torsions.

\begin{figure}[t]
    \centering
    \includegraphics[width=0.99\textwidth]{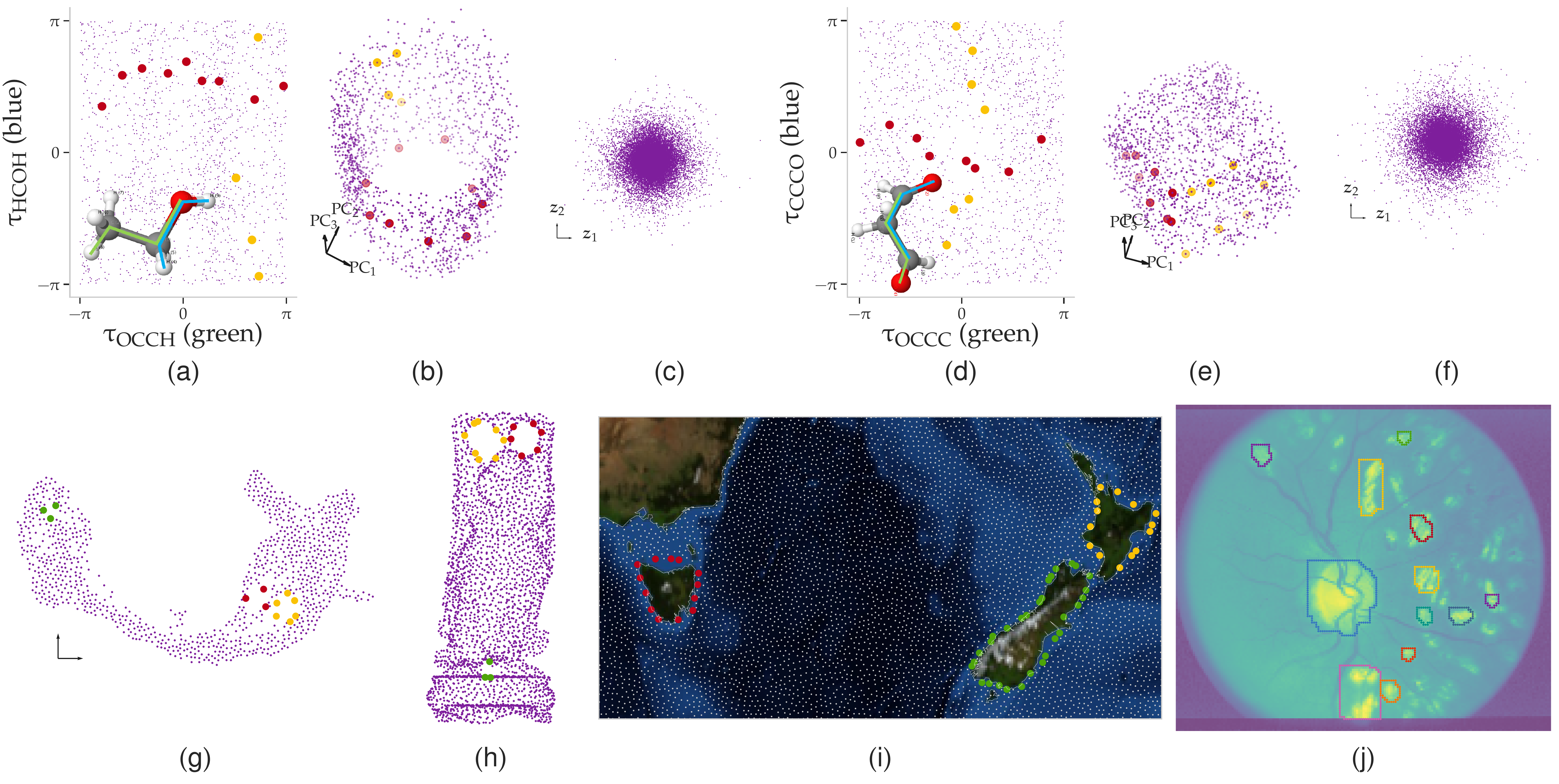}
    \caption{(a) and (b) are the detected loops of \datethanol{} using Dijkstra on $\vZ$ (in (c)) in the torsion space (inset of (a)) and in the PCA space, respectively. (d)--(f) are the results for \datmda{} that are similar to those for \datethanol{} in (a)--(c). (g)--(j) show the identified loops using $\vZ$ for \datrna, \datbudd, \datisland, and \datretina, respectively.}
    \label{fig:real-datasets}
\end{figure}

\paragraph{RNA single-cell sequencing data \cite{BergenV.L.P+:20}.}
The {\em trajectory inference} methods \cite{SaelensW.C.T+:19} for
analyzing the RNA single-cell sequencing datasets aim to order the cells
(points in high-dimensional expression space) along developmental
trajectories, which are inferred from the structure of the point clouds.
Identifying loops in the dataset can serve as a building block for
delineating a correct trajectory, especially for determining
cell cycle and cell differentiation. To illustrate the idea,
we compute the 1-Laplacian on the CkNN kernel \cite{BerryT.S:19}
constructed on the UMAP \cite{McInnesL.H.M+:18} embedding (Algorithm 
\ref{alg:ica-no-prewhite}). Figure
\figref{fig:detected-cycles-rna} shows the identified
loops from Algorithm \ref{alg:dijkstra-shortest-loop}, with the green loop 
being the cycle of ductal cells 
and yellow/red loops representing a trifurcation (endocrine cell
differentiation).

\paragraph{Additional point cloud datasets.}
\datbudd{} \cite{CurlessB.L:96} is a 3D model of a Buddha statue
with a precomputed triangulation. We treat the 3D model as a point cloud and 
subsample 3000 farthest points from the original dataset;  $\LL_1$
is obtained from the VR complex of the CkNN kernel. Note that with this small sample size, two smaller loops near the waist of the statue are not detectable.
Hence, the number of zero eigenvalues of $\LL_1$ is $3$, with
the corresponding homology generators shown in Figure \figref{fig:detected-cycles-3d-surface}.
\datisland{} \cite{FroylandG.P:15}, which contains ocean buoys around the Tasman sea, 
is the other point cloud in our analysis.
The estimated $\beta_1$ is $3$, with the detected loops being the North Island
of New Zealand, the South Island of New Zealand, and the main island of Tasmania
(Figure \figref{fig:detected-cycles-islands-aus-nz}).

\paragraphbf{Non-manifold dataset.}
Our framework for identifying subspaces is still valid for 
cubical complexes built from images (by Corollary 
\ref{thm:subspace-perturb-bound-Lk-cub}).
We demonstrate the idea
on \datretina{}, a medical retinal image
\cite{HooverA.G:03}. The cubical complex is constructed by intensity 
thresholding (also called the sub-level set method in TDA \cite{WassermanL:18}) 
and then applying morphological 
closing on the binary image to remove small cavities. The
weight for every rectangle $\vw_2(\sigma)$ is set to 1; the estimated 
null space dimension of the $\LL_1$ built from $\cub$ is $\beta_1 = 12$, 
with the identified homologous loops in Figure \figref{fig:detected-cycles-retina}.
The result shows the robustness of the proposed framework even for large $\beta_1$.

\section{Conclusion}
\label{sec:discussion-conclusion}
Our contributions in the emerging field of spectral algorithms for 
$k$-Laplacians $\LL_k$ 
\cite{ChenYC.M.K+:21,SchaubMT.B.H+:20,LimLH:20,EbliS.S:19} are summarized as follows.
(i) We extend the study of the homology embedding of vertices
by the graph Laplacian $\LL_0$ (spectral clustering) to those of higher-order
simplices by $\LL_k$.
Specifically, the $k$-th homology embedding can be approximately 
factorized into parts,
with each corresponding to a prime manifold given a small perturbation
(small $\epsilon_\ell$ and $\epsilon'_\ell$ for $\ell = k, k-1$). 
(ii) The analysis is 
made possible by expressing the $\kappa$-fold {\em connected sum} as a 
matrix perturbation. This convenient property of the homology embedding
supports (iii) the use of ICA to identify each decoupled subspace and motivates
(iv) the application to the 
shortest homologous loop detection problem.

Our analysis provides insight into the structure of the $k$-th harmonic 
embedding. This framework can inspire researchers in developing {\em
  spectral} topological data analysis algorithms (e.g., visualization,
clustering, tightest higher-order {\em cycles} for $k\geq 2$
\cite{DeyTK.H.M+:20,ObayashiI:17}) similar to those that were
inaugurated by spectral clustering two decades ago.  These
applications are especially beneficial to scientists (chemists,
biologists, oceanographers, etc.) who use high-dimensional data
analysis techniques for studying complex systems. 
Similar to the limitation of other unsupervised learning algorithms, practitioners without 
solid understandings of {\em both} the analyzed datasets and the used 
algorithm 
might draw controversial conclusions (see, e.g., discussions in
\cite{NovembreJ.S:08,Alquicira-HernandezJ.P.P+:20}).
Possible approaches to mitigate the negative consequences are to design
proper validation and causal inference algorithms
for this framework;
we leave them as potential directions we will explore.

\section*{Acknowledgements}
The authors acknowledge partial support from the U.S. Department of Energy's Office of Energy Efficiency and Renewable Energy (EERE) under the Solar Energy Technologies Office Award Number DE-EE0008563 and from
the National Science Foundation award DMS 2015272.
They thank the Tkatchenko
and Pfaendtner labs and in particular to Stefan Chmiela and Chris Fu for
providing the molecular
dynamics data and for many hours of brainstorming and advice.

\section*{Disclaimer}
The views expressed herein do not necessarily represent the views of the U.S. Department of Energy or the United States Government.

%\bibliography{ref_autogen,hhd_paper_ref,manifolds,mmp}
\bibliographystyle{\usedbibstyle}

\setupsupp

\section{Proof of subspace perturbations (Theorem \ref{thm:subspace-perturb-bound-Lk})}
\label{sec:proof-subspace-perturbation}

\subsection{A formal version of Assumption \ref{assu:degree-of-simplex-increase-less}}

\begin{assumption}
	Let $\tilde\vw_k = |\vB_{k+1}[\fN_k, \fN_{k+1}]| \vw_{k+1}$, $\tilde\vw_{k-1} = |\vB_{k}[:, \fN_{k}]|\tilde\vw_{k}$, 
	with $\vw_k$ and $\hat\vw_k$ defined in Section
	\ref{sec:prob-formulation}. Additionally, write

	\begin{equation*}
	\begin{gathered}
		\vW_{k+1} = \tilde\vW_{k+1} + \EE_{k+1, +}, \\
		\hat\vW_{k+1} = \tilde\vW_{k+1} + \EE_{k+1, -}, \\
		{\vW}_k^{1/2} = \tilde{\vW}_k^{1/2}(\vI + \vE^{+}_{k,+}) + \EE^{1/2}_{k,+}, \\
		{\vW}_k^{-1/2} = \tilde{\vW}_k^{-1/2}(\vI - \vE^{-}_{k,+}) + \EE^{-1/2}_{k,+}, \\
		\hat{\vW}_k^{1/2} = \tilde{\vW}_k^{1/2}(\vI + \vE^{+}_{k,-}) + \EE^{1/2}_{k,-}, \\
		\hat{\vW}_k^{-1/2} = \tilde{\vW}_k^{-1/2}(\vI - \vE^{-}_{k,-}) + \EE^{-1/2}_{k,-},\\
		{\vW}_{k-1}^{-1/2} = \tilde{\vW}_{k-1}^{-1/2}(\vI - \vE_{k-1,+}), \\
		\hat{\vW}_{k-1}^{-1/2} = \tilde{\vW}_{k-1}^{-1/2}(\vI - \vE_{k-1,-}).
	\end{gathered}
	\end{equation*}

	There exists $\epsilon_\ell>0$ and $\epsilon'_\ell>0$ for $\ell = k, k-1$ 
	such that the following conditions hold
	\begin{enumerate}
		\item Not too many $(k+1)$-simplices are created (small $|\fC_{k+1}|$)
		\begin{subequations}
		\begin{align}
			\|\vE_{k,+}^+\| = \max_{\sigma\in \fN_k}\left\{ \left[\vE_{k,+}^+\right]_{\sigma, \sigma}\right\} \,&=\, \max_{\sigma\in {\fN_k}} \left\{ \frac{w_k^{1/2}(\sigma)}{\tilde{w}_k^{1/2}(\sigma)} - 1\right\} \leq \sqrt{\epsilon_k}; \\
			\|\vE_{k,+}^-\| = \max_{\sigma\in \fN_k}\left\{ \left[\vE_{k,+}^-\right]_{\sigma, \sigma}\right\} \,&=\, \max_{\sigma\in \fN_k} \left\{\frac{\tilde{w}_k^{-1/2}(\sigma)}{w_k^{-1/2}(\sigma)} - 1\right\} \leq \sqrt{\epsilon_k}; \\
			\,&\pequal\,\max_{\sigma\in \fN_k} \left\{\frac{w_k(\sigma)}{\tilde{w}_k(\sigma)} - 1\right\} \leq \epsilon_k; \\
			\|\vE_{k-1,+}\| = \max_{\nu\in \Sigma_{k-1}}\left\{\left[\vE_{k-1,+}\right]_{\nu,\nu}\right\} \,&=\, \max_{\nu\in \Sigma_{k-1}} \left\{ \frac{\tilde{w}^{-1}_{k-1}(\nu)}{w^{-1}_{k-1}(\nu)} - 1 \right\} \leq \sqrt{\epsilon_{k-1}}; \\
			\,&\pequal\,\max_{\nu\in \Sigma_{k-1}} \left\{ \frac{{w}_{k-1}(\nu)}{\tilde{w}_{k-1}(\nu)} - 1\right\} \leq \epsilon_{k-1}.
		\end{align}
		\end{subequations}
		\item Not too many $(k+1)$-simplices are destroyed (small $|\fD_{k+1}|$)
		\begin{subequations}
		\begin{align}
			\|\vE_{k,-}^+\| = \max_{\sigma\in \fN_k}\left\{ \left[\vE_{k,-}^+\right]_{\sigma, \sigma}\right\} \,&=\, \max_{\sigma\in {\fN_k}} \left\{ \frac{\hat{w}_k^{1/2}(\sigma)}{\tilde{w}_k^{1/2}(\sigma)} - 1\right\} \leq \sqrt{\epsilon_k}; \\
			\|\vE_{k,-}^-\| = \max_{\sigma\in \fN_k}\left\{ \left[\vE_{k,-}^-\right]_{\sigma, \sigma}\right\} \,&=\, \max_{\sigma\in \fN_k} \left\{\frac{\tilde{w}_k^{-1/2}(\sigma)}{\hat{w}_k^{-1/2}(\sigma)} - 1\right\} \leq \sqrt{\epsilon_k}; \\
			\,&\pequal\,\max_{\sigma\in \fN_k} \left\{\frac{\hat{w}_k(\sigma)}{\tilde{w}_k(\sigma)} - 1\right\} \leq \epsilon_k; \\
			\|\vE_{k-1,-}\| = \max_{\nu\in \Sigma_{k-1}}\left\{\left[\vE_{k-1,-}\right]_{\nu,\nu}\right\} \,&=\, \max_{\nu\in \Sigma_{k-1}} \left\{ \frac{\tilde{w}^{-1}_{k-1}(\nu)}{\hat{w}^{-1}_{k-1}(\nu)} - 1 \right\} \leq \sqrt{\epsilon_{k-1}}; \\
			\,&\pequal\,\max_{\nu\in \Sigma_{k-1}} \left\{ \frac{\hat{w}_{k-1}(\nu)}{\tilde{w}_{k-1}(\nu)} - 1\right\} \leq \epsilon_{k-1}.
		\end{align}
		\end{subequations}
		\item The {\em net} changes on $\vw_k$ and $\vw_{k-1}$ are small
		\begin{subequations}
		\begin{align}
			\|\vE_{k,+}^+ - \vE_{k,-}^+\| \,&=\, \max_{\sigma\in {\fN_k}} \left\{ \left|\frac{\hat{w}_k^{1/2}(\sigma)}{{w}_k^{1/2}(\sigma)} - 1\right|\right\} \leq \sqrt{\epsilon'_k}; \\
			\|\vE_{k,+}^- - \vE_{k,-}^-\| \,&=\, \max_{\sigma\in {\fN_k}} \left\{ \left|\frac{{w}_k^{1/2}(\sigma)}{\hat{w}_k^{1/2}(\sigma)} - 1\right|\right\} \leq \sqrt{\epsilon'_k}; \\
			\,&\pequal\,\max_{\sigma\in \fN_k} \left\{\left|\frac{\hat{w}_k(\sigma)}{{w}_k(\sigma)} - 1\right|\right\} \leq \epsilon'_k; \\
			\|\vE_{k-1,+}-\vE_{k-1,-}\| \,&=\, \max_{\nu\in \Sigma_{k-1}} \left\{\left| \frac{{w}^{-1}_{k-1}(\nu)}{\hat{w}^{-1}_{k-1}(\nu)} - 1\right| \right\} \leq \sqrt{\epsilon'_{k-1}}; \\
			\,&\pequal\,\max_{\nu\in \Sigma_{k-1}} \left\{\left| \frac{\hat{w}_{k-1}(\nu)}{{w}_{k-1}(\nu)} - 1\right|\right\} \leq \epsilon'_{k-1}.
		\end{align}
		\end{subequations}
	\end{enumerate}

\label{assu:degree-of-simplex-increase-less-formal}
\end{assumption}

\subsection{Definitions of $\mathbf{\mathcal L}_k$ and $\hat{\mathbf{\mathcal L}}_k$}

Given a manifold $\M$ which is constructed by a series of connected sum, i.e.,
$\M = \M_1\sharp\cdots\sharp\M_\kappa$. Let the 
Simplicial complex corresponding to $\M$
be
$\scx_\ell = (\Sigma_0,\cdots,\Sigma_\ell)$, with the disjoint simplicial complex (of $\cup_{i=1}^\kappa \M_i$)
being
$\hat{\scx}_\ell = (\hat\Sigma_0, \cdots, \hat\Sigma_\ell)$.
For each $k$, the simplex sets can be decomposed into the following

\begin{equation*}
	\Sigma_k = \underbrace{\bigcup_{i=1}^\kappa \Sigma_k^{(i)}}_{\text{non-intersecting set:} \fN_k} \cup \underbrace{\bigcup_{j> i}^\kappa \Sigma_k^{(ij)+}}_{\text{created set:} \fC_k}.
\end{equation*}

Similarly,
\begin{equation*}
	\hat\Sigma_k = \underbrace{\bigcup_{i=1}^\kappa \Sigma_k^{(i)}}_{\text{non-intersecting set:} \fN_k} \cup \underbrace{\bigcup_{j> i}^\kappa \Sigma_k^{(ij)-}}_{\text{destroyed set:} \fD_k}.
\end{equation*}

W.l.o.g., one can assume that the $(k-1)$-simplices set can
be perfectly separated, i.e.,
$\fC_{k-1} = \fD_{k-1} = \emptyset$ (when analyzing the $k$-Laplacian).
The above construction matches our intuition; by definition,
a connected sum is a process of carving out a $d$-disk
($\fD_k$) and gluing two manifolds together ($\fC_k$).

We are interested in the perturbation of the $k$-Laplacian $\LL_k$ w.r.t. 
the ideal (disjoint) Laplacian $\hat\LL_k$. Without carefully define both $\LL_k$ and
$\hat\LL_k$, the perturbation on the subspaces might be unbounded.
With slight abuse of notation, we let $\vL \gets \LL_k$, $\vL_d \gets \LL_k^{\rm down}$, and $\vL_u \gets \LL_k^{\rm up}$ (similar definitions
for $\hat\LL$'s). The $k$ is omitted and can be inferred from the context. 
The $\hat\vL$ and $\vL$ are defined as follows.
$\hat{\vL}$ is a block diagonal matrix, with the $i$-th (diagonal) block
$\vL^{(i)}$ described by $\M_i$ constructed from the sub-complex $\hat\scx^{(i)}$ 
($\hat\Sigma_{k-1}^{(i)}$, $\hat\Sigma_k^{(i)}$, and $\hat\Sigma_{k+1}^{(i)}$). 
Due to manifolds being disjoint (i.e., $\cup_{i=1}^\kappa \M_i$), the 
Laplacian corresponding to such block, denoted $\hat{\vL}^{(i,i),(i,i)}$, 
will be a valid Laplacian. 
As for the intersecting $k$-simplices $\fC_k\cup\fD_k$, we let $\hat{\vL}^{(i,j),(k,l)} = \vL^{(i,j),(k,l)}$ for all 
$ij, kl \in \binom{[k]}{2}$ so that
the corresponding blocks of $\hat{\vL} - \vL$ will be zero. 
Under this scenario, the unbounded
increase of $(k+1)$-simplices caused by the intersecting $k$-simplices
can be removed. Lastly, the off-diagonal blocks of $\hat{\vL}$
are set to zero. Specifically, $\hat\vL$ is,

\begin{adjustwidth}{-90pt}{90pt}
\begin{equation*}
	\hat\vL = \begin{bmatrix}
		\begin{array}{ccc|ccc|ccc}
		\hat\vL^{(1,1),(1,1)} & & & & & & \hat\vL^{(1,1), (1,2)-} & \cdots & \hat\vL^{(1,1), (k-1,k)-}\\
		& \ddots & & & \vec 0 & & \vdots & \ddots & \vdots \\
		& & \hat\vL^{(k,k),(k,k)} & & & & \hat\vL^{(k,k), (1,2)-} & \cdots & \hat\vL^{(k,k), (k-1,k)-}\\
		\hline
		& & & \vL^{(1,2)+,(1,2)+} & \cdots & \vL^{(1,2)+,(k-1,k)+} & & &\\
		& \vec 0 & & \vdots & \ddots & \vdots  & & \vec 0 &\\
		& & & \vL^{(k-1,k)+,(1,2)+} & \cdots & \vL^{(k-1, k)+,(k-1, k)+} & & & \\
		\hline
		\hat\vL^{(1,2)-,(1,1)} & \cdots & \hat\vL^{(1, 2)-, (k,k)} & & & & \hat\vL^{(1,2)-,(1,2)-} & \cdots & \hat\vL^{(1,2)-,(k-1,k)-}\\
		\vdots & \ddots & \vdots &  & \vec 0 & & \vdots & \ddots & \vdots\\
		\hat\vL^{(k-1,k)-,(1,1)} & \cdots & \hat\vL^{(k,k)-, (k,k)} & & & & \hat\vL^{(k-1,k)-,(1,2)-} & \cdots & \hat\vL^{(k-1, k)-,(k-1, k)-}
		\end{array}
	\end{bmatrix}.
\end{equation*}
\end{adjustwidth}

Similarly, one can define $\vL$ to be

\begin{adjustwidth}{-90pt}{90pt}
\begin{equation*}
	\vL = \begin{bmatrix}
		\begin{array}{ccc|ccc|ccc}
		\vL^{(1,1),(1,1)} & & & \vL^{(1,1), (1,2)+} & \cdots & \vL^{(1,1), (k-1,k)+}& & &\\
		& \ddots & & \vdots & \ddots & \vdots & & \vec 0 & \\
		& & \vL^{(k,k),(k,k)} & \vL^{(k,k), (1,2)+} & \cdots & \vL^{(k,k), (k-1,k)+}& & &\\
		\hline
		\vL^{(1,2)+,(1,1)} & \cdots & \vL^{(1, 2)+, (k,k)} & \vL^{(1,2)+,(1,2)+} & \cdots & \vL^{(1,2)+,(k-1,k)+} & & &\\
		\vdots & \ddots & \vdots & \vdots & \ddots & \vdots  & & \vec 0 &\\
		\vL^{(k-1,k)+,(1,1)} & \cdots & \vL^{(k,k)+, (k,k)} & \vL^{(k-1,k)+,(1,2)+} & \cdots & \vL^{(k-1, k)+,(k-1, k)+} & & & \\
		\hline
		& & & & & & \hat\vL^{(1,2)-,(1,2)-} & \cdots & \hat\vL^{(1,2)-,(k-1,k)-}\\
		 & \vec 0 & &  & \vec 0 & & \vdots & \ddots & \vdots\\
		& & & & & & \hat\vL^{(k-1,k)-,(1,2)-} & \cdots & \hat\vL^{(k-1, k)-,(k-1, k)-}
		\end{array}
	\end{bmatrix}.
\end{equation*}
\end{adjustwidth}

Under this construction, the four lower right blocks, which correspond to 
the $k$-simplices in $\fC_k\cup\fD_k$, will be zero. 
If no new homology class is created/destroyed (Assumption \ref{assu:basic-ml-assu})
and the minimum eigenvalues of the last two diagonal blocks are bounded away from zero (Assumption \ref{assu:connected-sum-preserve-homology}),
then the eigengap of $\vL$ will simply be the minimum eigengap of each $\hat\vL^{(i)}$,
i.e., ${\rm eigengap}(\vL) = \min\{\delta_1,\cdots,\delta_\kappa\}$.

Now we formally define our formulation. 
Following the notations introduced in Section \ref{sec:prob-formulation},
and let $\cI_\sigma$ be the index set of the $k$-simplex $\sigma\in\fN_k$
sampled from $\M_i$. Note that $\cI_\sigma$ is defined only for $\sigma\in \fN_k$,
which can be extended from the index set
$\cI_v$ for $v\in V$ introduced in Section \ref{sec:prob-formulation}
by $\cI_\sigma = \{\sigma\in\fN_k: v\in \cI_v \text{ for } v\in \sigma\}$.
Note also that similar to $\cI_v$ for $V$,
$\cS_\sigma$ can be larger than 1.
For instance, if the manifold is constructed by gluing a torus (indexed by 1)
and a circle (indexed by 2), then $\cS_{1} = \{1, 2\}$ and $\cS_2 = \{3\}$;
for an edge $e$ belongs to the torus, we have $\cS_{\cI_e} = \{1, 2\}$.
For every $\sigma\in\fN_k$, we write,
\begin{equation*}
	\sum_{\sigma\in\fN_k} \sum_{i\notin S_{\cI_\sigma}} \vY^2_{\sigma, i} \leq \sum_{\sigma\in\fN_k} \sum_{i=1}^{\beta_1} ( \vY_{\sigma,i} - \hat{\vY}_{\sigma,i} )^2 \leq \sum_{\sigma\in \Sigma_k\cup \hat\Sigma_k}\sum_{i=1}^{\beta_1} (\vY_{\sigma,i} - \hat{\vY}_{\sigma,i} )^2 = \| \vY \vO - \hat{\vY}\|_F^2.
\end{equation*}

Let $\difflap_k^{\rm down} = \vL_d - \hat\vL_d$ and $\difflap_k^{\rm up} = \vL_u - \hat\vL_u$,
from \cite{YuY.W.S+:15} and the triangular inequality,
\begin{equation*}
\begin{split}
	\left\|\vY_{\fN_k, :} - \hat\vY_{\fN_k, :}\right\|_F^2 \,&= \, \sum_{\sigma\in\fN_k} \sum_{i\notin \cS_{\cI_\sigma}} \vY^2_{\sigma, i} \leq \|\vY - \hat\vY\vO \|_F^2 \\
	\,&\leq\, \frac{8\cdot \min\left\{\beta_k\left\| \vL - \hat{\vL} \right\|^2, \|\vL - \hat{\vL}\|_F^2 \right\} }{\min\{\delta_1,\cdots,\delta_\kappa \}} \\
	\,&\leq\, \frac{8\cdot \min\left\{\beta_k\left\| \difflap_k^{\rm down} \right\|^2+\beta_k\left\| \difflap_k^{\rm up} \right\|^2, \|\difflap_k^{\rm down}\|_F^2 + \|\difflap_k^{\rm up}\|_F^2 \right\} }{\min\{\delta_1,\cdots,\delta_\kappa \}} \\
	\,&\stackrel{\dagger}{\leq}\, \frac{8\beta_k \left(\left\| \difflap_k^{\rm down} \right\|^2 + \left\| \difflap_k^{\rm up} \right\|^2\right)}{\min\{\delta_1,\cdots,\delta_\kappa\}}
	.
\end{split}
\end{equation*}

\begin{remark}
	The bound w.r.t. the Frobenius norm is omitted (the last inequality $\dagger$) 
	based on two reasons:
	(i) $\LL_k$ has complicated forms for large $k$, therefore, it is hard
	to derive a concise expression; and (ii) $\|\cdot \|_F$ is usually
	larger than $\beta_k\|\cdot\|$.	
\end{remark}

\subsection{Useful lemmas}
Here we omit the $k$ for $\fN$, $\fC$, and $\fD$ for simplicity.
Let $\lambda_k = \|\LL_k\|$ be the bound on the spectral norm of
$k$-Laplacian. Here, $\lambda_k = k+2$ for $\LL$'s built from
simplicial complexes; $\lambda_k = 2k+2$ for those built from
cubical complexes (see also Proposition \ref{thm:spectral-norm-Lk-cb}).
The following two lemmas bound the effects of $\EE_{k,+}$, 
$\EE_{k,-}$, $\EE_{k+1,+}$, and $\EE_{k+1,-}$ in their changes
to the weights ($\vW_k$ and $\vW_{k-1}$) of the $k$ and $(k-1)$-simplices; 
we will find them useful in proving Theorem \ref{thm:subspace-perturb-bound-Lk}.

\begin{lemma}
Let $\vW_k$, $\hat\vW_k$, $\EE_{k,+}$, and $\EE_{k,-}$ defined in Assumption \ref{assu:degree-of-simplex-increase-less-formal}, we have
\begin{equation*}
\begin{gathered}
	\left\|\EE_{k,+}\vB_k^\top {\vW}_{k-1}^{-1}\vB_k\EE_{k,+}\right \|\leq \lambda_{k-1}\epsilon_{k-1}, \\
	\left\|\EE_{k,-}\vB_k^\top \hat{\vW}_{k-1}^{-1}\vB_k\EE_{k,-}\right\| \leq \lambda_{k-1}\epsilon_{k-1}.	
\end{gathered}
\end{equation*}
\label{thm:bound-of-mutrally-exclusive-edges}
\end{lemma}

\begin{proof}
We first inspect the case of $\fC$, i.e., the first equation involving $\EE_{k,+}$,
\begin{equation*}
	[\EE_{k,+}]_{\sigma,\sigma} = \begin{cases}
		w_k^{1/2}(\sigma) \,&\,\text{ if } \sigma \in \fC; \\
		0 \,&\,\text{ otherwise.}
	\end{cases}
\end{equation*}

for any $\nu\in \Sigma_{k-1}$, we have,

\begin{equation*}
\begin{split}
	w_{k-1}(\nu) \,&=\, |\vB_k(\nu)|\vw_k; \\
	\tilde{w}_{k-1}(\nu) \,&=\, |\vB_k(\nu)|\tilde{\vw}_k.	
\end{split}
\end{equation*}

Therefore, 
\begin{equation*}
\begin{split}
	\epsilon_{k-1} w_{k-1}(\nu) \,&\geq\, \epsilon_{k-1} \tilde w_{k-1}(\nu) \geq w_{k-1}(\nu) - \tilde{w}_{k-1}(\nu) = |\vB_k(\nu)|(\vw_k - \tilde{\vw}_k) \\
	\,&=\, |\vB_k(\nu)|\left[\tilde\vw_k\vE_k + \EE_{k,+} \right] \geq |\vB_k(\nu)| \EE_{k,+} = \deg(\nu).
\end{split}
\end{equation*}

Let $f_m$ be the $k$-eigencochain corresponding to the largest eigenvalue of $\EE_{k,+}\vB_k^\top\vW_{k-1}^{-1}\vB_k\EE_{k,+}$.
From Eq. (3.6) of \cite{HorakD.J:13}, we have,

\begin{equation*}
\begin{split}
	\|\EE_{k,+}\vB_k^\top\vW_{k-1}^{-1}\vB_k\EE_{k,+}\|_2 \,&\leq\, \|\vL_d\|\cdot \frac{\sum_{\nu\in \Sigma_{k-1}} f_m^2(\nu) \deg(\nu) }{\sum_{\nu\in\Sigma_{k-1}} f_m^2(\nu) w_{k-1}(\nu)} \\
	\,&\leq\, \lambda_{k-1}\epsilon_{k-1} \cdot \frac{\sum_{\nu\in\Sigma_{k-1}} f_m^2(\nu) w_{k-1}(\nu)}{\sum_{\nu\in\Sigma_{k-1}} f_m^2(\nu) w_{k-1}(\nu)} = \lambda_{k-1}\epsilon_{k-1}.	
\end{split}
\end{equation*}

The case of $\fD$ follows similarly. 
\end{proof}

The following lemma bounds the changes in $(k+1)$-simplices
with $\epsilon_k$.

\begin{lemma}
	Let $\vW$ be either $\vW_k$ or $\hat\vW_k$, and
	$\EE$ be either $\EE_{k+1,+}$ or $\EE_{k+1,-}$ defined
	in Assumption \ref{assu:degree-of-simplex-increase-less-formal},
	we have
	\begin{equation*}
		\left\|\vW\vB_{k+1}\EE\vB_{k+1}^\top\vW \right\| \leq \lambda_k\epsilon_k.
	\end{equation*}
\label{thm:bound-of-mutrally-exclusive-triangles}
\end{lemma}
\begin{proof}
	Consider the case of $\vW_{k}$ and $\EE_{k+1,+}$. For any 
	$\sigma\in \Sigma_{k}$,
	\begin{equation*}
	\begin{split}
		w_k(\sigma) \,&=\, |\vB_{k+1}(\sigma)|\vw_{k+1}; \\
		\tilde{w}_k(\sigma) \,&=\, |\vB_{k+1}(\sigma)|\tilde{\vw}_{k+1}.
		\end{split}
	\end{equation*}
	
	Therefore, for any $\sigma \in \fN$ (do not count the one in $\EE_{k, \pm}$)
	we have,
	\begin{equation*}
		\epsilon_k w_k(\sigma) \geq \epsilon_k \tilde w_k(\sigma) \geq w_k(\sigma) - \tilde{w}(\sigma) = |\vB_{k+1}(\sigma)|(\vw_{k+1} - \tilde{\vw}_{k+1}) = |\vB_{k+1}(\sigma)|\EE_{k+1,+}.
	\end{equation*}
	
	Let $f_m$ be the $k$-eigencochain corresponding to the largest eigenvalue of the matrix $\vW_k^{-1/2}\vB_{k+1}\EE_{k+1,+}\vB_{k+1}^\top\vW_k^{-1/2}$.
	From Eq. (3.6) of \cite{HorakD.J:13},
	\begin{equation*}
	\begin{split}
		\left\|\vW_k^{-1/2}\vB_{k+1}\EE_{k+1,+}\vB_{k+1}^\top\vW_k^{-1/2} \right\|\,&\leq\, (k+2)\cdot \frac{\sum_{\sigma\in\fN} f_m^2(\sigma) \deg(\sigma) }{\sum_{\sigma\in\fN} f_m^2(\sigma) w_k(\sigma)} \\
		\,&\leq\, \lambda_k \epsilon_k \frac{\sum_{\sigma\in\fN} f_m^2(e) w_k(e) }{\sum_{\sigma\in\fN} f_m^2(e) w_k(e)} = \lambda_k\epsilon_k	
	\end{split}
	\end{equation*}
	
	Here $\deg(\sigma) = |\vB_{k+1}(\sigma)\diag(\EE_{k+1,+})|$.
	Consider the case when $\vW \gets \hat\vW_k$ and $\EE\gets \EE_{k+1,+}$,
	we have,
	\begin{equation*}
		\epsilon_k \hat w_k(\sigma) \geq \epsilon_k \tilde w_k(\sigma) \geq  w_k(\sigma) - \tilde{w}(\sigma) = |\vB_{k+1}(\sigma)|(\vw_{k+1} - \tilde{\vw}_{k+1}) = |\vB_{k+1}(\sigma)|\EE_{k+1,+}.
	\end{equation*}
	The result follows similarly for $\EE\gets\EE_{k+1,-}$;
	this completes the proof.
\end{proof}

\subsection{Proof of Theorem  \ref{thm:subspace-perturb-bound-Lk}}
Now we start the formal proof of Theorem 
\ref{thm:subspace-perturb-bound-Lk}. We will break the proof into two
parts, i.e., the {\em down} and {\em up} parts involving 
$\difflap_k^{\rm down}$ and $\difflap_k^{\rm up}$, respectively. 

\begin{proofof}{the $\difflap_k^{\rm down}$ term in Theorem \ref{thm:subspace-perturb-bound-Lk}}
The explicit form of the {\em down} Laplacian can be written as
\begin{equation*}
	\hat{\vL}_d = 
	\begin{bmatrix}
	\begin{array}{c|c|c}
		\vM_\fN \hat{\vW}_k^{1/2}\vB_k^\top\hat{\vW}_{k-1}^{-1} \vB_k \hat{\vW}_k^{1/2} \vM_\fN & \vec 0 & \vM_\fN \hat{\vW}_k^{1/2}\vB_k^\top \hat\vW_{k-1}^{-1} \vB_k {\EE}_{k,-}^{1/2} \vM_\fD \\
		\hline
		\vec 0 & \vM_\fC \EE_{k,+}^{1/2}\vB_k^\top{\vW}_{k-1}^{-1} \vB_k \EE_{k,+}^{1/2} \vM_\fC & \vec 0 \\
		\hline
		\vM_\fD {\EE}_{k,-}^{1/2} \vB_k^\top \hat\vW_{k-1}^{-1} \vB_k \hat{\vW}_k^{1/2} \vM_\fN & \vec 0 & \vM_\fD \EE_{k,-}^{1/2}\vB_k^\top\hat{\vW}_{k-1}^{-1} \vB_k \EE_{k,-}^{1/2} \vM_\fD
	\end{array}
	\end{bmatrix}.
\end{equation*}

And,

\begin{equation*}
	{\vL}_d = 
	\begin{bmatrix}
	\begin{array}{c|c|c}
		\vM_\fN {\vW}_k^{1/2}\vB_k^\top \vW_{k-1}^{-1} \vB_k {\vW}_k^{1/2} \vM_\fN & \vM_\fN {\vW}_k^{1/2}\vB_k^\top \vW_{k-1}^{-1} \vB_k {\EE}_{k,+}^{1/2} \vM_\fC & \vec 0 \\
		\hline
		\vM_\fC {\EE}_{k,+}^{1/2}\vB_k^\top \vW_{k-1}^{-1} \vB_k {\vW}_k^{1/2} \vM_\fN & \vM_\fC \EE_{k,+}^{1/2}\vB_k^\top\vW_{k-1}^{-1} \vB_k \EE_{k,+}^{1/2} \vM_\fC & \vec 0 \\
		\hline
		\vec 0 & \vec 0 & \vM_\fD \EE_{k,-}^{1/2}\vB_k^\top\hat{\vW}_{k-1}^{-1} \vB_k \EE_{k,-}^{1/2} \vM_\fD
	\end{array}
	\end{bmatrix}.
\end{equation*}

Here, $\vM_\fN$, $\vM_\fC$, and $\vM_\fD$ are diagonal masks for $k$-simplex sets $\fN$, $\fC$, and $\fD$, respectively.
By triangular inequality, 
\begin{equation}
\begin{split}
	\left\| \vL_d - \hat{\vL}_d \right\|  \,&\leq\, \underbrace{\left\| \vM_\fN {\vW}_k^{1/2}\vB_k^\top \vW_{k-1}^{-1} \vB_k {\vW}_k^{1/2} \vM_\fN - \vM_\fN \hat{\vW}_k^{1/2}\vB_k^\top \hat{\vW}_{k-1}^{-1} \vB_k \hat{\vW}_k^{1/2} \vM_\fN \right\|}_{(*)} + \\
	\,&\pequal\, 2\left[\underbrace{\left\| \vM_\fN {\vW}_k^{1/2}\vB_k^\top \vW_{k-1}^{-1} \vB_k {\EE}_{k,+}^{1/2} \vM_{\fC}\right\|}_{(\dagger)} + \left\| \vM_\fN \hat{\vW}_k^{1/2}\vB_k^\top \hat\vW_{k-1}^{-1} \vB_k {\EE}_{k,-}^{1/2} \vM_{\fD} \right\|\right].
\end{split}
\label{eq:down-Laplacian-block-bound}
\end{equation}

Expand the $\vW_k$ with $\hat\vW_k$ and omit $\vM_\fN$ for simplicity, 
the first term of \eqref{eq:down-Laplacian-block-bound} can be bounded by
\begin{equation*}
\begin{split}
	\text{(*)} \,&\leq\, \bigg\| \hat\vW_k^{1/2}\left(\vI+ (\vE_{k,+}^+ - \vE_{k,-}^+)\right)\vB_k^\top \hat\vW_{k-1}^{-1}\left(\vI - (\vE_{k-1,+} - \vE_{k-1,-}) \right) \vB_k\hat\vW_k^{1/2}\left(\vI + (\vE_{k,+}^+ - \vE_{k,-}^+) \right) \\
	\,&\pequal\, - \hat\vW_k^{1/2}\vB_k^\top \hat\vW_{k-1}^{-1} \vB_k\hat\vW_k^{1/2} \bigg\| \\
	\,&\leq\, \Bigg[ \left( \left\|2\cdot\left(\vE^+_{k,+} - \vE^+_{k,-}\right) \right\| + \left\| \left(\vE^+_{k,+} - \vE^+_{k,-}  \right)^2  \right\| \right) \cdot \left\|\hat\vL_d\right\| +   \\
	\,&\pequal\, \phantom{\Bigg[} \left(1+\sqrt{\epsilon_k'}\right)^2 \left\| \hat\vW_k^{1/2}\vB_k^\top\hat\vW_{k-1}^{-1}\left(\vE_{k-1, +} - \vE_{k-1, -} \right) \vB_k\hat\vW_k^{1/2} \right\| \Bigg] \\
	\,&\leq\, \left[\left\|2\cdot\left(\vE^+_{k,+} - \vE^+_{k,-}\right) \right\|+\left\| \left(\vE^+_{k,+} - \vE^+_{k,-}  \right)^2  \right\| + \left(1+\sqrt{\epsilon_k'}\right)^2\left\| \vE_{k-1,+}-\vE_{k-1,-} \right\| \right]\cdot \|\hat\vL_d\| \\
	\,&\stackrel{*}{\leq}\,	\left[ 2\sqrt{\epsilon_k'} + \epsilon_k' + \left(1+\sqrt{\epsilon_k'}\right)^2 \sqrt{\epsilon_{k-1}'}  \right] \cdot\left\|\tilde\vL_d\right\|.
\end{split}
\end{equation*}

The last two terms of \eqref{eq:down-Laplacian-block-bound}
can be bounded using Lemma \ref{thm:bound-of-mutrally-exclusive-edges}, i.e.,
\begin{equation*}
\begin{split}
	(\dagger) = \left\| \vM_\fN {\vW}_k^{1/2}\vB_k^\top \vW_{k-1}^{-1} \vB_k {\EE}_{k,+}^{1/2} \vM_{\fC}\right\|  \,&\leq\, \left\|\vM_\fN\vW_k^{1/2}\vB_k^\top\vW_{k-1}^{-1/2}\right\| \cdot \left\| \vW_{k-1}^{-1/2} \vB_k\EE_{k,+}^{1/2}\vM_\cI \right\|  \\
	\,&\leq\,  \left\|\vL_d\right\| \sqrt{\epsilon_{k-1}}.
\end{split}
\end{equation*}

The last term of \eqref{eq:down-Laplacian-block-bound} can also be bounded
by $\|\tilde\vL_d\|\sqrt{\epsilon_{k-1}}$ using Lemma \ref{thm:bound-of-mutrally-exclusive-edges}.
Since $\left\|\vL_d\right\|$, $\left\|\hat\vL_d\right\|$, and $\left\|\tilde\vL_d\right\|$
have the same upper bound $\lambda_{k-1}$, we have

\begin{equation*}
	\left\|\vL_d - \hat\vL_d\right\|^2 \leq \left[2\sqrt{\epsilon_k'} + \epsilon_k' + \left(1+\sqrt{\epsilon_k'}\right)^2 \sqrt{\epsilon_{k-1}'} + 4\sqrt{\epsilon_{k-1}}\right]^2\lambda^2_{k-1}.
\end{equation*}
\end{proofof}

\begin{proofof}{the $\difflap_k^{\rm up}$ term in Theorem \ref{thm:subspace-perturb-bound-Lk}}

The explicit form of $\hat\vL_u$ is,

\begin{adjustwidth}{-90pt}{90pt}
\begin{equation*}
	\hat{\vL}_u = 
	\begin{bmatrix}
	\begin{array}{c|c|c}
		\vM_\fN \hat{\vW}_k^{-1/2}\vB_{k+1}\hat{\vW}_{k+1} \vB_{k+1}^\top \hat{\vW}_k^{-1/2} \vM_\fN & \vec 0 & \vM_\fN \hat{\vW}_k^{-1/2}\vB_{k+1} \hat\vW_{k+1} \vB_{k+1}^\top {\EE}_{k,-}^{1/2} \vM_\fD \\
		\hline
		\vec 0 & \vM_\fC \EE_{k,+}^{-1/2}\vB_{k+1} {\vW}_{k+1} \vB_{k+1}^\top \EE_{k,+}^{-1/2} \vM_\fC & \vec 0 \\
		\hline
		\vM_\fD {\EE}_{k,-}^{-1/2} \vB_{k+1} \hat\vW_{k+1} \vB_{k+1}^\top \hat{\vW}_k^{-1/2} \vM_\fN & \vec 0 & \vM_\fD \EE_{k,-}^{-1/2}\vB_{k+1}\hat{\vW}_{k+1} \vB_{k+1}^\top \EE_{k,+}^{-1/2} \vM_\fD
	\end{array}
	\end{bmatrix}.
\end{equation*}
\end{adjustwidth}

And,

\begin{adjustwidth}{-90pt}{90pt}
\begin{equation*}
	{\vL}_u = 
	\begin{bmatrix}
	\begin{array}{c|c|c}
		\vM_\fN {\vW}_k^{-1/2}\vB_{k+1}{\vW}_{k+1} \vB_{k+1}^\top {\vW}_k^{-1/2} \vM_\fN  & \vM_\fN {\vW}_k^{-1/2}\vB_{k+1} \vW_{k+1} \vB_{k+1}^\top {\EE}_{k,+}^{1/2} \vM_\fC & \vec 0\\
		\hline
		\vM_\fC {\EE}_{k,+}^{-1/2} \vB_{k+1} \vW_{k+1} \vB_{k+1}^\top {\vW}_k^{-1/2} \vM_\fN & \vM_\fC \EE_{k,+}^{-1/2}\vB_{k+1} {\vW}_{k+1} \vB_{k+1}^\top \EE_{k,+}^{-1/2} \vM_\fC & \vec 0 \\
		\hline
		\vec 0 & \vec 0 & \vM_\fD \EE_{k,-}^{-1/2}\vB_{k+1}\hat{\vW}_{k+1} \vB_{k+1}^\top \EE_{k,+}^{-1/2} \vM_\fD
	\end{array}
	\end{bmatrix}.
\end{equation*}
\end{adjustwidth}

The perturbation is,

\begin{equation}
\begin{split}
	\left\| \vL_u - \hat{\vL}_u \right\|  \,&\leq\, \underbrace{\left\| \vM_\fN {\vW}_k^{-1/2}\vB_{k+1}{\vW}_{k+1} \vB_{k+1}^\top {\vW}_k^{-1/2} \vM_\fN - \vM_\fN \hat{\vW}_k^{-1/2}\vB_{k+1}\hat{\vW}_{k+1} \vB_{k+1}^\top \hat{\vW}_k^{-1/2} \vM_\fN \right\|}_{(*)} + \\
	\,&\pequal\, 2\left[\underbrace{\left\| \vM_\fN {\vW}_k^{-1/2}\vB_{k+1} \vW_{k+1} \vB_{k+1}^\top {\EE}_{k,+}^{1/2} \vM_\fC \right\|}_{(\dagger)} + \left\| \vM_\fN \hat{\vW}_k^{-1/2}\vB_{k+1} \hat\vW_{k+1} \vB_{k+1}^\top {\EE}_{k,-}^{1/2} \vM_\fD\right\|\right].
\end{split}
\label{eq:up-Laplacian-block-bound}
\end{equation}

The first term of \eqref{eq:up-Laplacian-block-bound}
can be bounded 
by expanding $\vW_{k+1}$ w.r.t. $\hat\vW_{k+1}$, i.e.,
$\vW_{k+1} = \hat\vW_{k+1} + (\EE_{k+1,+} - \EE_{k+1,-})$.
As slight abuse of
notation, we let $\vW_k \gets \vW_k[\fN,\fN]$, $\vB_{k+1}\gets \vB_{k+1}[\fN, :]$.
The first term $(*)$ of \eqref{eq:up-Laplacian-block-bound} becomes 

\begin{equation*}
\begin{split}
	(*) \,&\leq\, \left\|{\vW}_k^{-1/2}\vB_{k+1} \hat\vW_{k+1} \vB_{k+1}^\top {\vW}_k^{-1/2} - \hat{\vW}_k^{-1/2}\vB_{k+1}\hat{\vW}_{k+1} \vB_{k+1}^\top \hat{\vW}_k^{-1/2} \right\| 	\\
	\,&\leq\, \left\| {\vW}_k^{-1/2}\vB_{k+1} \hat{\vW}_{k+1} \vB_{k+1}^\top {\vW}_k^{-1/2} - \hat{\vW}_k^{-1/2}\vB_{k+1}\hat{\vW}_{k+1} \vB_{k+1}^\top \hat{\vW}_k^{-1/2} \right\| + \\
	\,&\pequal\, \left\| {\vW}_k^{-1/2}\vB_{k+1} {\EE}_{k+1,+} \vB_{k+1}^\top {\vW}_k^{-1/2} - \hat{\vW}_k^{-1/2}\vB_{k+1} {\EE}_{k+1,-} \vB_{k+1}^\top \hat{\vW}_k^{-1/2} \right\| \\
	\,&\stackrel{\ddag}{\leq}\, \left( 2\|\vE_{k,+}^- - \vE_{k,-}^-\| + \left\|\left(\vE_{k,+}^- - \vE_{k,-}^-\right)^2\right\| \right) \cdot \left\| \hat\vL_u  \right\| + \\
	\,&\pequal\, \left\| {\vW}_k^{-1/2}\vB_{k+1} {\EE}_{k+1,+} \vB_{k+1}^\top {\vW}_k^{-1/2} - \hat{\vW}_k^{-1/2}\vB_{k+1} {\EE}_{k+1,-} \vB_{k+1}^\top \hat{\vW}_k^{-1/2} \right\| \\
	\,&\stackrel{\mathsection}{\leq} \left[  2\sqrt{\epsilon'_k} + \epsilon'_k + 2\epsilon_k \right] \lambda_k
\end{split}
\end{equation*}

The $\ddag$ term holds by expanding ${\vW}_k^{-1/2} = \hat{\vW}_1^{-1/2}\left(\vI - (\vE_{k,+}^- - \vE_{k, -}^-)\right)$ and following
a similar approach of the {\em down} Laplacian.
The $\mathsection$ term holds by bounding $\vE_{k,+}^- - \vE^-_{k,-}$
with Assumption \ref{assu:degree-of-simplex-increase-less-formal} ($\epsilon'_k$)
and using Lemma \ref{thm:bound-of-mutrally-exclusive-triangles} ($\epsilon_k$).

The $(\dagger)$ term in \eqref{eq:up-Laplacian-block-bound} can be bounded by $\epsilon_k$ using Lemma \ref{thm:bound-of-mutrally-exclusive-triangles},
i.e.,

\begin{equation*}
\begin{split}
	(\dagger) \,&\stackrel{\ddag}{=}\, \left\| \vM_\fN {\vW}_k^{-1/2}\vB_{k+1} \EE_{k+1,+} \vB_{k+1}^\top {\EE}_{k,+}^{-1/2} \vM_\fC\right\| \\
	\,&\leq\, \left\| \vM_\fN {\vW}_k^{-1/2}\vB_{k+1} \EE_{k+1,+}^{1/2}\right\| \cdot \left\| \EE_{k+1,+}^{1/2} \vB_{k+1}^\top {\EE}_{k,+}^{-1/2}\right\| \\
	\,&\leq\, \sqrt{\lambda_k\epsilon_k} \cdot \left\| \EE_{k+1,+}^{1/2} \vB_{k+1}^\top {\EE}_{k,+}^{-1/2}\right\| \\
	\,&\stackrel{\mathsection}{\leq}\, \sqrt{\epsilon_k} \lambda_k.
\end{split}
\end{equation*}

$\ddag$ holds because the intersection of triangles of $\EE_{k,+}$, and $\vW_k$ is the triangles with non-zero entries in $\EE_{k+1,+}$.
$\mathsection$ holds (the $\sqrt{\lambda_k}$ term) because $\EE_{k+1,+}^{1/2} \vB_{k+1}^\top {\EE}_{k,+}^{-1/2}$
is a submatrix of $\vW_{k+1}^{1/2} \vB_{k+1}^\top {\vW}_k^{-1/2}$; hence, the spectral
norm will be upper bounded by the {\em up} Laplacian $\|\vL_u\| \leq \lambda_k$. 

Therefore, we have
\begin{equation*}
	\left\|\vL_u - \hat\vL_u\right\|^2 \leq \left[2\sqrt{\epsilon_k'} + \epsilon_k' + 2\epsilon_k + 4\sqrt{\epsilon_{k}}\right]^2\lambda^2_{k}.
\end{equation*}

Combining the bound involving $\difflap_k^{\rm down}$ completes the proof of 
Theorem \ref{thm:subspace-perturb-bound-Lk}.
\end{proofof}

 \section{Proofs of propositions in Applications (Section \ref{sec:applications})}
\label{sec:other-proofs}
\label{sec:proofs-of-applications-props}

\subsection{Proof of Proposition \ref{thm:digraph-reachable-no-shortcuit}: the properties of the induced digraph}

The proof is based on the convenient properties
of the harmonic flow (the basis of the homology vector space), i.e.,
they are both {\em divergence-free} and {\em curl-free}
\cite{LimLH:20,ChenYC.M.K+:21,SchaubMT.B.H+:20}.

\begin{proofof}{Proposition \ref{thm:digraph-reachable-no-shortcuit}}
	\paragraphbf{Reachable:} the harmonic flow is divergence free, indicating that the incoming flow must be equal to
	the outgoing flow. If there exists a vertex that is not {\em reachable} to itself, then
	this vertex will either be a {\em source} or {\em sink} in the digraph. It violates the assumption that
	the flow is divergence free. Therefore such vertex will not exist.

	\paragraphbf{No short-circuiting:} the harmonic flow is curl free; from Stoke's theorem (or Poincar\'e Lemma \cite{LeeJM:13}),
	we have that any path-integral travel along any homology class will be a constant. If there exists
	a loop such that it does not traverse along with any homology class, the loop integral along this
	cycle will be zero (by Stoke's theorem). By assumption, the path-integral will always be positive.
	To generate a loop whose integral is zero, one has to travel ``upward'' in the
	digraph; this violates the assumption
	that we are finding a cycle in the digraph, implying that every loop will traverse along at least one homology class.
\end{proofof}

\subsection{Proof of Proposition \ref{thm:harmonic-emb-of-Tk}: $\cH_1$ embedding of $\mathbb T^m$}

The proof is based on the fact that each harmonic 1-form of the flat 
$m$-(flat) torus can be expressed as the $m$-dimensional standard basis
multiplied with some intensities in the intrinsic coordinate.
The closed-form of the upper bound of the embedding distribution in any 
direction can be derived using the (high-dimensional) polar coordinate system, 
indicating  that the envelope is an $m$-dimensional ellipsoid. The detailed
proof is provided below.

\begin{proofof}{Proposition \ref{thm:harmonic-emb-of-Tk}}
The harmonic vector field in an $m$-flat torus  $\mathbb T^m$ is a constant
in each coordinate, i.e., $\vv = [v_1,\cdots,v_m]\in \rrr^m$. The manifold
$\mathbb T^m$ is an $m$-dimensional cube with the periodic boundary condition,
i.e., $0 = 2\pi$.
From \cite{WhitneyH:05,ChenYC.M.K+:21}, the edge flow $\vec\omega_e$ for
an edge $e=(i, j) \in E$ can be written exactly as a linear map, i.e.,
\begin{equation*}
\begin{split}
	\omega_e \,&=\, \int_0^1 \vv^\top(\vec\gamma(t))\gamma'(t) \dd t = \int_0^1 [\vv(\vx_i) + (\vv(\vx_j) - \vv(\vx_i)) t]^\top (\vx_j - \vx_i)\dd t \\
	\,&=\,	\inv{2} (\vv(\vx_i) + \vv(\vx_j))^\top(\vx_j - \vx_i)
\end{split}
\end{equation*}

Where $\vec\gamma(t)$ is the geodesic on $\M$ connecting $\vx_i$ and $\vx_j$
with $\vec\gamma(0) = \vx_i$ and $\vec\gamma(1) = \vx_j$.
Any point $\vx\in\rrr^m$, with $r = \|\vx\|$, can 
be written as $\vx = [r f_1(\vec\Phi),r f_2(\vec\Phi),\cdots, rf_m(\vec\Phi)]$, 
where $\vec\Phi\in\rrr^{m-1}$ is the high-dimensional polar coordinate;
for instance, a point in 2D is  $[r\cos(\theta), r\sin(\theta)]$ with $\vec\Phi = [\theta]$, 
while a point in 3D having $\vec\Phi = [\theta,\varphi]$ is $[r\cos\varphi\sin\theta, r\sin\varphi\sin\theta, r\cos\theta]$.
The conditional distribution given a fixed $\vec\Phi$ is simply the distribution
of edge lengths, i.e., $p(rv_1f_1,\cdots,rv_mf_m | \vec\Phi) = p(r)$.
Since $p(r)$ is bounded by some constant $\delta$ representing the maximum 
edge length, the envelope of the distribution is bounded by 
$[\delta v_1 f_1(\vec\Phi),\cdots, \delta v_m f_m(\vec\Phi)]$,
indicating that it is an $m$-ellipsoid with the length of the $i$-th semi-axes 
being $\delta v_i$.
\end{proofof}

 \section{The maximum eigenvalue of $\mathcal L_k$ constructed from a cubical complex}
\label{sec:max-eval-cubical-cpx}

In this section, we would like to show the bound on the spectral norm of $\LL_k$
built from a cubical complex. The property is found useful in extending
Theorem \ref{thm:subspace-perturb-bound-Lk} to Corollary
\ref{thm:subspace-perturb-bound-Lk-cub}; namely, the goal is to show that
$\|\LL_k\|_2 \leq \lambda_k = (2k + 2)$. Note that 
$\|\LL_k^{\rm down}\|= \|\vA_k^\top\vA_k\| = \|\vA_k\vA_k^\top\| = \|\LL_{k-1}^{\rm up}\|$.
W.l.o.g., one can inspect only the up-Laplacian. We provide the following
proposition that is largely based on the similar analysis \cite{HorakD.J:13}
of $\|\LL_k\|$ for $\scx$.

\begin{proposition}
	Given an {\em up} $k$-Laplacian $\LL_k^{\rm up} = \vA_{k+1}\vA_{k+1}^\top$
	with $\vA_{k+1} = \vW_k^{-1/2}\vB_{k+1}\vW_{k+1}^{1/2}$ built from a cubical 
	complex, we have 
	\begin{equation*}
		\|\LL_k^{\rm up}\|_2 \leq \lambda_k = 2k+2.
	\end{equation*}
\label{thm:spectral-norm-Lk-cb}
\end{proposition}

\begin{proof}
From \cite{SchaubMT.B.H+:20},
	the eigenvalues of the $k$-th renormalized {\em up}-Laplacian $\LL_k^{\rm up}$
	are identical to those of the $k$-th
	{\em random-walk} {\em up}-Laplacian $\LL_k^{\rm rw} = \vW_k^{-1}\vB_{k+1}\vW_{k+1}\vB_{k+1}$.
	Further, let $\vL_k^{\rm up} = \vB_{k+1}\vW_{k+1}\vB_{k+1}^\top$, following
	the analysis of \cite{HorakD.J:13}, we have
	
	\begin{equation*}
	\begin{split}
		\vf^\top \vL_k^{\rm up} \vf \,&=\, \left( \vW_{k+1}^{1/2} \vB_{k+1}^\top \vf \right)^\top \left( \vW_{k+1}^{1/2} \vB_{k+1}^\top \vf \right) \\
		\,&=\, \sum_{\sigma\in K_k}\sum_{\tau\in\mathrm{coface}(\sigma)} f^2(\sigma) w_{k+1}(\tau) \\
		\,&\stackrel{\dagger}{\leq}\, (2k+2)\sum_{\sigma\in K_k} f^2(\sigma) \sum_{\tau\in\mathrm{coface}(\sigma)} w_{k+1}(\tau) \\
		\,&=\, (2k+2)\sum_{\sigma\in K_k} f^2(\sigma) \deg(\sigma).
	\end{split}
	\end{equation*}
	
	The inequality $\dagger$ holds using the Cauchy-Schwarz inequality;
	the $2k+2$ term comes from the fact that a $(k+1)$-cube has $(2k+2)$ 
	faces. Following the rest of the proof in \cite{HorakD.J:13}, we have
	
	\begin{equation*}
		\|\LL_k^{\rm up}\| = \|\LL_k^{\rm rw, up}\| = \frac{\|\vL_k^{\rm up}\|}{ \vf^\top \vW_k \vf } \leq (2k+2) \frac{\sum_{\sigma\in K_k} f^2(\sigma) \deg(\sigma) }{\sum_{\sigma\in K_k} f^2(\sigma) w_k(\sigma)} = 2k+2.
	\end{equation*}
	
	The first equality holds due to the identical eigenvalues of $\LL_k$ and
	$\LL_k^{\rm rw}$; 
	the last inequality holds because we have $\vw_k(\sigma) = |\vB_{k+1}(\sigma)|\vw_{k+1} = \deg(\sigma)$ for all $\sigma\in K_k$.
\end{proof}
 \section{Datasets and experiment details}
\label{sec:supp-datasets-detail}
The edge set $E$ of the neighborhood graph constructed using the CkNN kernel 
\cite{BerryT.S:19} is

\begin{equation*}
	E = \left\{ i,j\in V: \frac{\|\vx_i - \vx_j\|}{\sqrt{\rho_k(\vx_i)\rho_k(\vx_j)}} \leq \delta  \right\}.
\end{equation*}

Here, $\rho_k(\vx)$ is the distance from $\vx$ to its $k$-th nearest neighbor;
throughout the experiment, we fix $k=30$.
The $\delta$ parameter can be chosen by a variant of the geometric consistent
(GC) algorithm \cite{JoncasD.M.M+:17} suitable for CkNN graphs;
for real datasets (except for the ocean drifter data whose geometric property 
is known),
we use the modified GC to choose this parameter.
For the rest of the datasets (synthetic manifolds and the ocean drifter),
$\delta$'s are chosen manually since the topologies are
known.
The weights on the triangles are selected by a modification to the kernel in
\cite{ChenYC.M.K+:21}, with a similar choice of $\varepsilon = \delta^{\frac{2}{3}} / 3$,
\begin{equation*}
	w_2(i, j, \ell) = \exp\left(-\frac{\|\vx_i - \vx_j\|^2}{\varepsilon \rho_k(\vx_i)\rho_k(\vx_j)} \right)\cdot \exp\left(-\frac{\|\vx_j - \vx_\ell\|^2}{\varepsilon \rho_k(\vx_j)\rho_k(\vx_\ell)} \right)\cdot \exp\left(-\frac{\|\vx_i - \vx_\ell\|^2}{\varepsilon \rho_k(\vx_i)\rho_k(\vx_\ell)} \right).
\end{equation*}

With this choice of parameters, the corresponding $\LL_1$ has 
a large sample size limit (in terms of $\Delta_1$) w.r.t. the metrics 
normalized by the k-nearest neighbor distance $\rho_k$.

\subsection{Synthetic manifolds}

\paragraph{\dattwoholes.}
\dattwoholes{} is a manifold generated by connected summing two punctured planes,
with a (sparsely connected) bridge in between. Each punctured plane has a 
rectangular hole with
width/height being $1/3$ of the width of each manifold.

\paragraph{\dattorus.}
This data is a two-dimensional torus and is generated from the parameterization 
below,

\begin{equation*}
\begin{split}
	x_1 \,&=\, (1 + 0.5\cos\theta_1)\cos\theta_2; \\
	x_2 \,&=\, (1 + 0.5\cos\theta_1)\sin\theta_2; \\
	x_3 \,&=\, 1 + 0.5\sin\theta_1.
\end{split}
\end{equation*}

The sample size is $n=1,156$.
Random Gaussian noise is added on the first three dimensions as well as
the additional 10 (noise) dimensions.

\paragraph{\datthreetorus.}
The parameterization of \datthreetorus, a three torus with $d=3$ and $D=4$,
is 

\begin{equation*}
\begin{split}
	x_1 \,&=\, (4 + (2 + \cos\theta_1)\cos\theta_2)\cos\theta_3; \\
	x_2 \,&=\, (4 + (2 + \cos\theta_1)\cos\theta_2)\sin\theta_3; \\
	x_3 \,&=\, (2 + \cos\theta_1)\sin\theta_2; \\
	x_4 \,&=\, \sin\theta_1.
\end{split}
\end{equation*}

We first sample $n'=100,000$ points from this manifold; Algorithm 
\ref{alg:furthest-point-sampling} is used to generate $\vX$ with $n = 2,000$.

\paragraph{\datgenustwo.}
\datgenustwo{} is a two-dimensional ({\em genus-2}) surface generated by gluing
two tori together. The implicit equation of the surface is

\begin{equation*}
	\left(\left(x_1^2 + x_2^2\right)^2 - 0.75 x_1^2 + 0.75x_2^2\right)^2 + x_3^2 = 0.01.
\end{equation*}

To sample from this surface, we create a $1,000\times1,000$ grid in the first two
coordinates ($x_1, x_2$) and solve for the corresponding $x_3$ from the above implicit
equation. The aforementioned procedure generates a point cloud  $\tilde\vX$ 
($n'\approx 551$k) having a
non-uniform sampling density on the genus-2 surface; we subsample $\tilde\vX$ by 
Algorithm \ref{alg:furthest-point-sampling} 
and obtain the final dataset $\vX$ with $n = 1,500$.

\paragraph{\datgenusfour.}
\datgenusfour{} is generated by concatenating four tori together. Four tori are generated
by similar procedures as \dattorus{} with horizontal movements (in $x_1$) being 
$a = -3, 0, 3, 6$, 
i.e., $x_1 = (1+0.5\cos\theta_1)\cos\theta_2 - a$,
respectively.
The sample size of \datgenusfour{} is $n = 4,624$.

\subsection{Real datasets}

\paragraph{Small molecule datasets (\datethanol{} and \datmda).}
The database \footnote{Data from \url{http://quantum-machine.org/datasets/}}
\cite{ChmielaS.T.S+:17} contains several molecular dynamics (MD) trajectories,
with each for a single (small) molecule, 
e.g., ethanol \ce{CH3CH2OH} (\datethanol) and malondialdehyde 
\ce{CH2(CHO)2} (\datmda). 
If a molecule has $N$ atoms, then a point (molecular configuration) in the dataset
is specified by an $N\times 3$ matrix representing the Euclidean coordinate of the
configuration. To generate a point cloud from a trajectory of configurations,
we first preprocess the data by calculating two angles of every triplet of atoms.
Secondly, we remove the linear subspaces by keeping the top {\em principal components} (PCs)
such that the unexplained variance ratio is less than $10^{-4}$.
The ambient dimensions of \datethanol{} and \datmda{} are $D = 102$ and $D=98$,
respectively. We
subsample furthest $n=1,500$ points using Algorithm \ref{alg:furthest-point-sampling}
for both datasets. The bond torsions (insets of Figures \figref{fig:detected-cycles-eth}
and \figref{fig:detected-cycles-mda}) are calculated by the dihedral angles of the
corresponding chemical bonds for each molecular configuration. For instance, the green 
torsion of ethanol (Figure \figref{fig:detected-cycles-eth}) for every point is computed 
by the angle of the planes spanned by \ce{OCC} and \ce{CCH} in the configuration
(3D Euclidean) space. One can think of the bond torsions as intrinsic coordinates
of \dattorus{}, i.e., $\theta_1$ and $\theta_2$; note that the correct bond torsions
parametrizing the manifold (or $\vX$) are usually unknown beforehand. In this work, this
information is provided based on our knowledge to validate our framework.

\paragraph{Single-cell RNA sequencing data \datrna.}
\datrna{} \cite{BergenV.L.P+:20} is a single-cell RNA sequencing data with cell
cycles. The data and preprocessing codes can be found in \url{https://github.com/theislab/scvelo_notebooks/blob/master/Pancreas.ipynb}.
The original data has sample size $n'=3,696$; we subsample $n=2,000$ furthest points
(Algorithm \ref{alg:furthest-point-sampling})
to remove the non-uniform sampling density on the original manifold.

\paragraph{3D graphics \datbudd.}
The 3D model of a Buddha statue, which can be downloaded from 
\url{https://www.cc.gatech.edu/projects/large_models/}, provides a triangulation
computed by \cite{CurlessB.L:96};
in other words, the simplicial complex $\scx'_2 = (V',E',T')$ is available 
beforehand, with $n'\approx 500$k and $n'_1\approx 2$M.
To illustrate the efficacy of our framework (and Theorem \ref{thm:subspace-perturb-bound-Lk}),
we treat \datbudd{} as a point cloud and build $\scx$ from the subsampled
$n = 3,000$ furthest points by Algorithm \ref{alg:furthest-point-sampling}.

\paragraph{Ocean buoys dataset \datisland.}
The global Lagrangian drifter data (available in \url{http://www.aoml.noaa.gov/envids/gld/})
was collected by NOAA's Atlantic Oceanographic and Meteorological Laboratory
and analyzed by \citet{FroylandG.P:15} on the coherent flow structures of the ocean
current. The dataset contains multiple trajectories of buoys dated between 2010--2019,
with the location, velocity, and water temperature of each buoy recorded.
The dataset itself is a 3D point cloud by converting the location 
(in latitude and longitude coordinates) to the {\em earth-centered, earth-fixed} 
(ECEF) coordinate system. We subsample $n=5,000$ furthest points/buoys (Algorithm
\ref{alg:furthest-point-sampling}) with longitudes within
$142^\circ$E--$179^\circ$E and latitudes between $48^\circ$S--$33^\circ$S; namely,
we sampled buoys that were located around the Tasman sea.

\paragraph{Medical imaging data \datretina.}
\datretina{} is one of the medical images of the STARE project \cite{HooverA.G:03},
a retinal imaging data collection. 
The database consists of around 400 raw images of human retinas, with diagnosis codes,
the segmented blood vessel, and the detected optic nerve available in 
\url{http://cecas.clemson.edu/~ahoover/stare/}.
We use the retinal
image with ID being 179, which has numerous bright (circular) spots visible. 
We construct the cubical complex by intensity thresholding and morphological closing,
resulting in $n=25,237$, $n_1=49,793$, and $n_2 = 24,548$.

\subsection{Pairwise scatter plots}
In this section, we show the pairwise scatter plots for $\vZ$ (blue)
and $\vY$ (red); specifically, we would like to show that the independent 
homology embedding $\vZ$ obtained by Algorithm \ref{alg:ica-no-prewhite}
is (approximately) factorizable. The blue embeddings (lower diagonal)
in Figures \ref{fig:pairplot-gen2}--\ref{fig:pairplot-retna}
confirm this. By contrast, most coordinate of the red embeddings $\vY$
do not correspond to a subspace, except for \datrna{} and \datbudd{} in 
Figures \ref{fig:pairplot-rna} and \ref{fig:pairplot-hpbud},
respectively.

\begin{figure}[!hbt]
    \centering
    \includegraphics[width=0.99\textwidth]{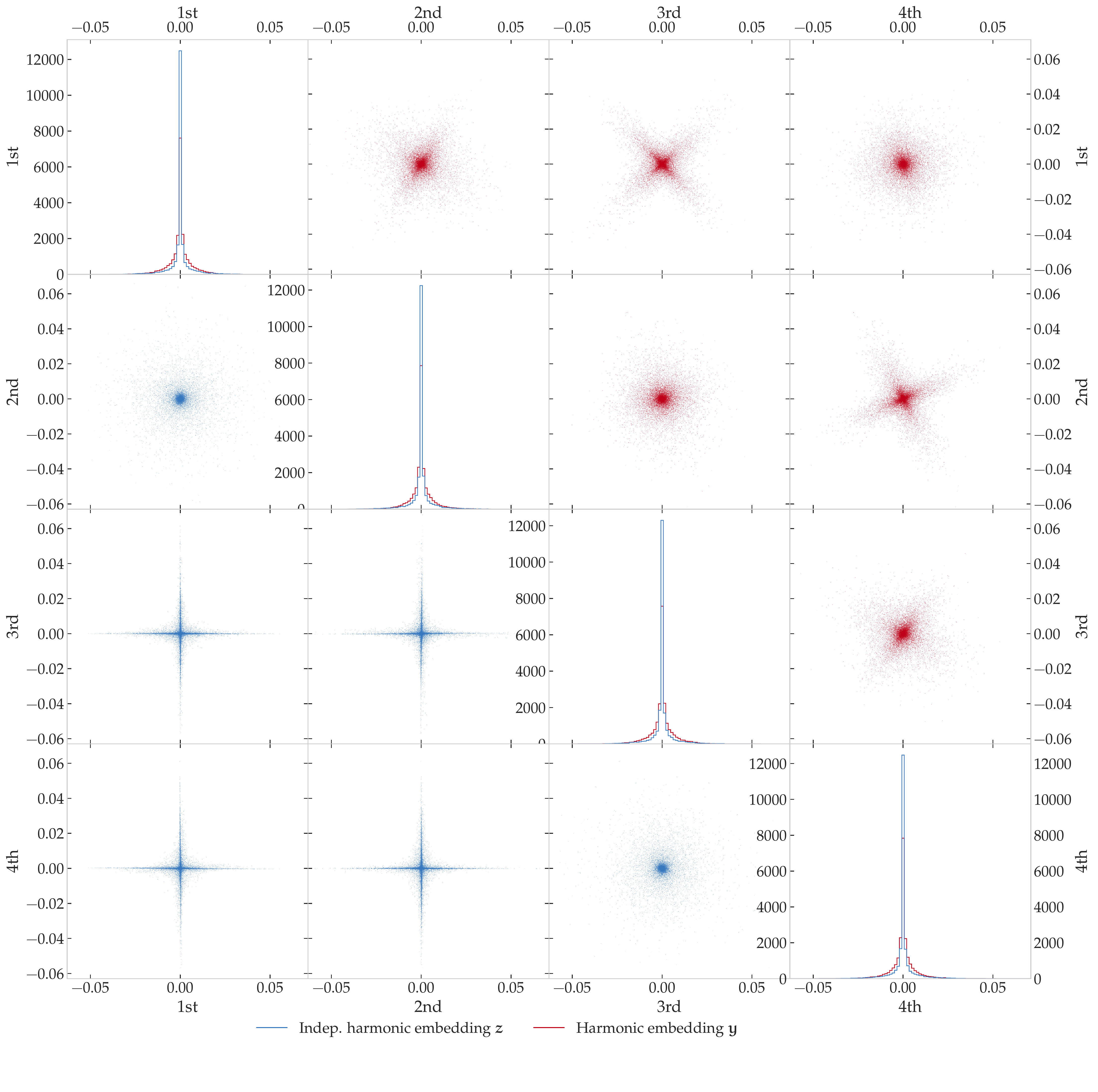}
    \caption{The independent ($\vz$, in blue) and the coupled ($\vy$, in red) homology embeddings of \datgenustwo{}. The $(i, j)$-th (off-diagonal) subplot represents the two-dimensional scatter plot with the $i$-th and $j$-th coordinates of the embedding; the $i$-th diagonal term is the histograms of the $i$-th coordinate of the corresponding embedding.}
    \label{fig:pairplot-gen2}
\end{figure}

\begin{figure}[!hbt]
    \centering
    \includegraphics[width=0.99\textwidth]{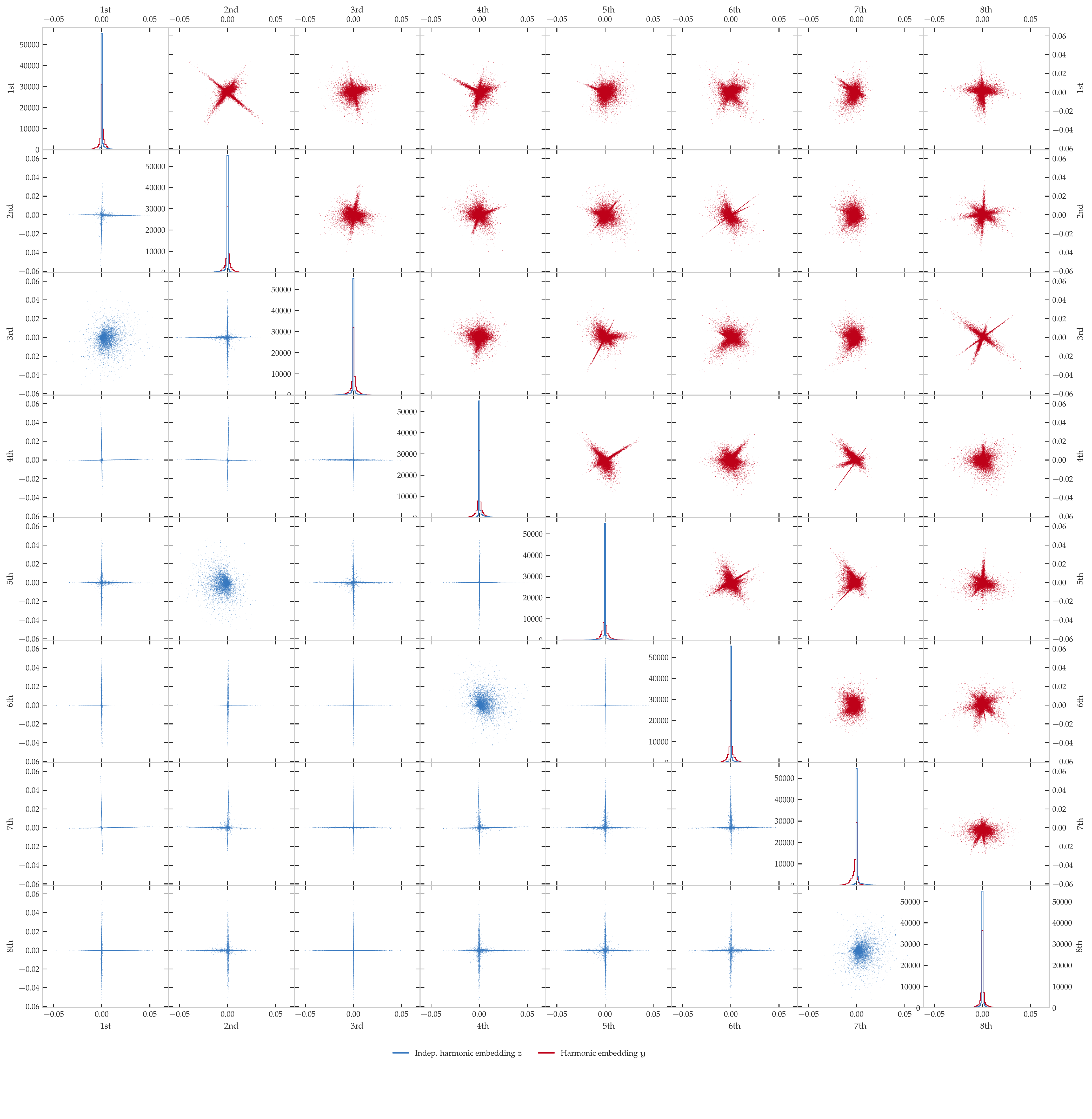}
    \caption{The independent ($\vz$, in blue) and the coupled ($\vy$, in red) homology embeddings of \datgenusfour{}.}
    \label{fig:pairplot-gen4}
\end{figure}

\begin{figure}[!hbt]
    \centering
    \includegraphics[width=0.99\textwidth]{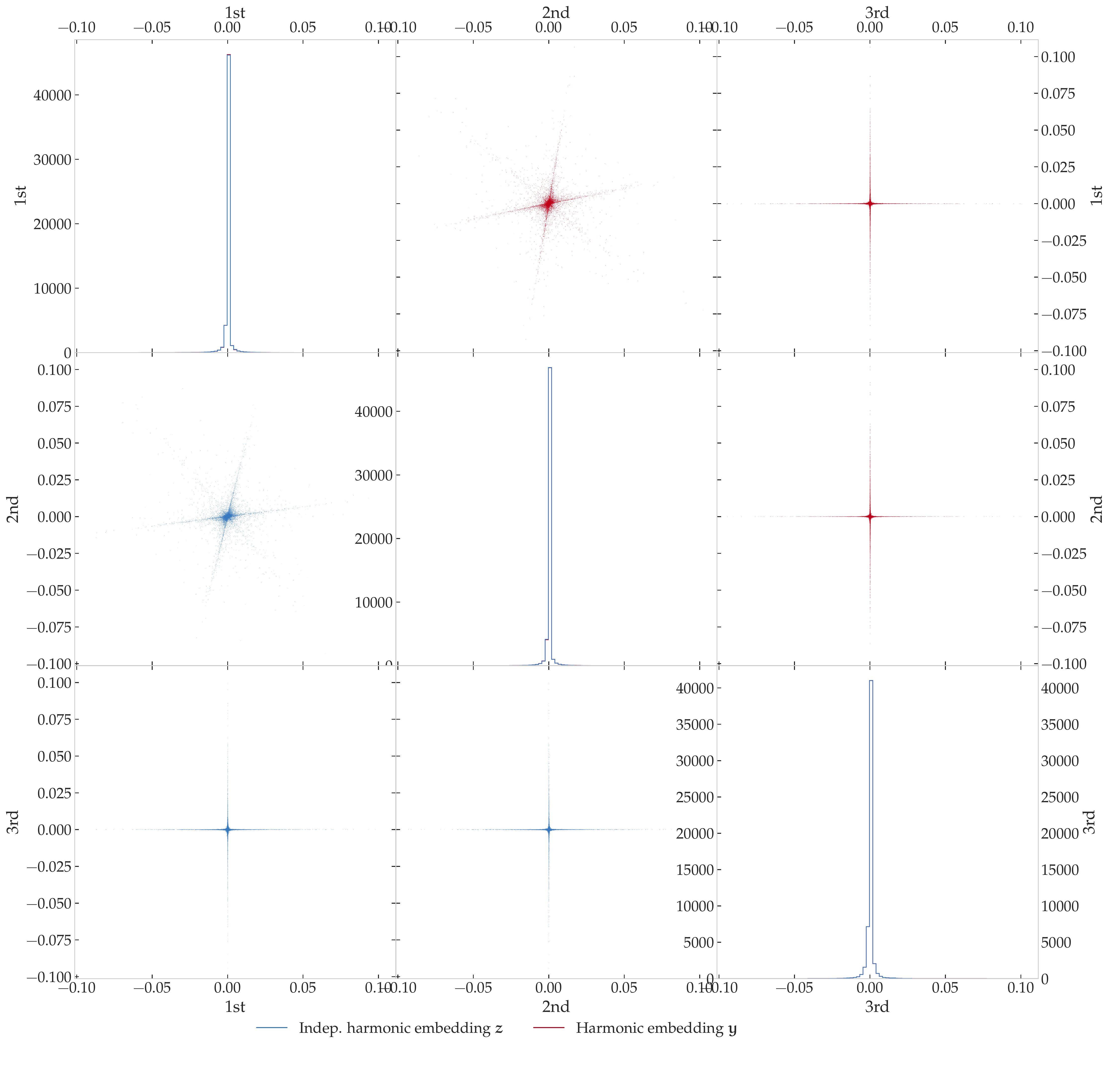}
    \caption{The independent ($\vz$, in blue) and the coupled ($\vy$, in red) homology embeddings of \datbudd{}.}
    \label{fig:pairplot-hpbud}
\end{figure}

\begin{figure}[!hbt]
    \centering
    \includegraphics[width=0.99\textwidth]{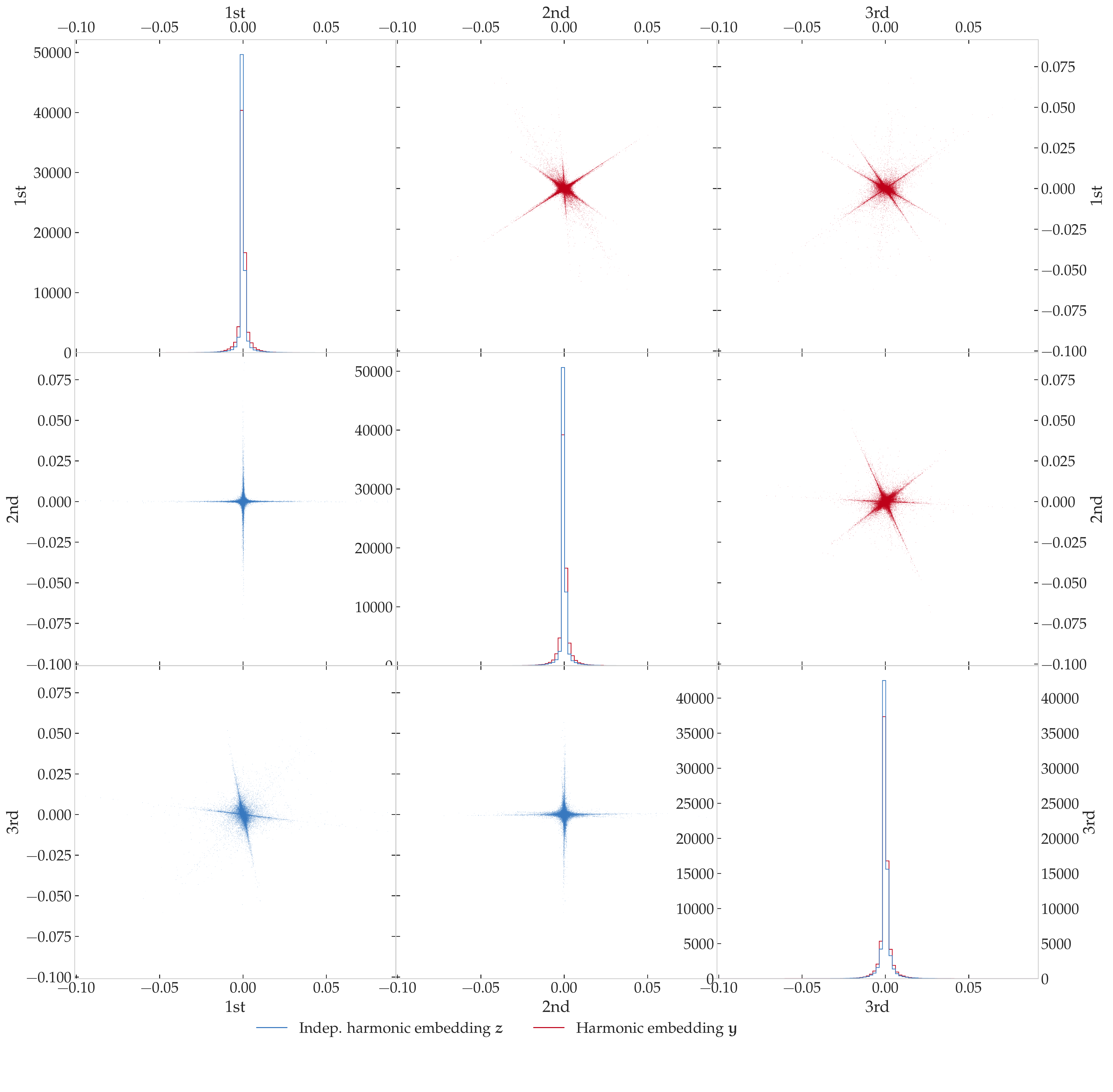}
    \caption{The independent ($\vz$, in blue) and the coupled ($\vy$, in red) homology embeddings of \datisland{}.}
    \label{fig:pairplot-islnaz}
\end{figure}

\begin{figure}[!hbt]
    \centering
    \includegraphics[width=0.99\textwidth]{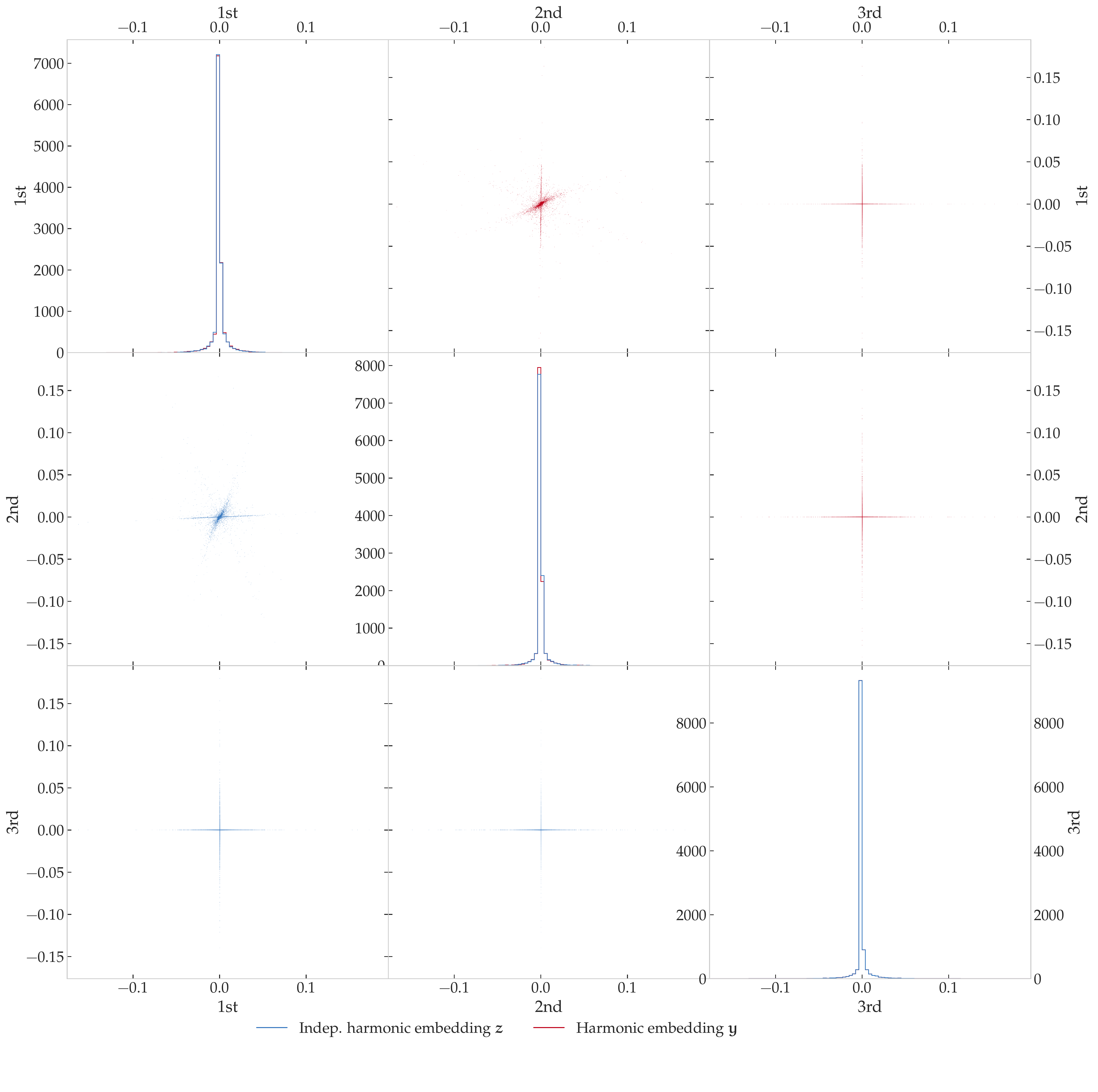}
    \caption{The independent ($\vz$, in blue) and the coupled ($\vy$, in red) homology embeddings of \datrna{}.}
    \label{fig:pairplot-rna}
\end{figure}

\begin{figure}[!hbt]
    \centering
    \includegraphics[width=0.99\textwidth]{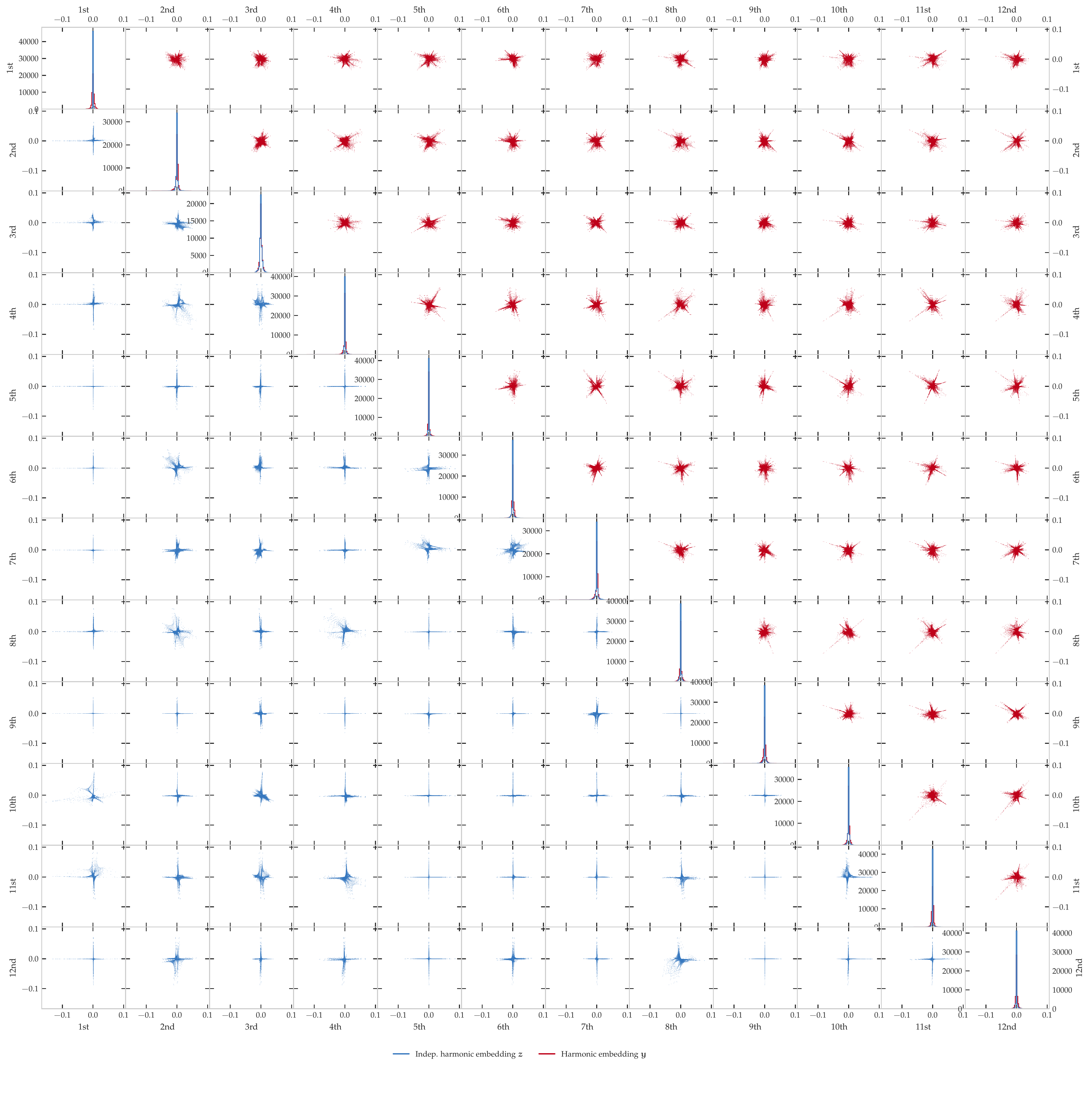}
    \caption{The independent ($\vz$, in blue) and the coupled ($\vy$, in red) homology embeddings of \datretina{}.}
    \label{fig:pairplot-retna}
\end{figure}

\clearpage

\subsection{Shortest homologous loops obtained from the coupled embedding $Y$}

\begin{figure}[t]
    \centering
    \includegraphics[width=0.99\textwidth]{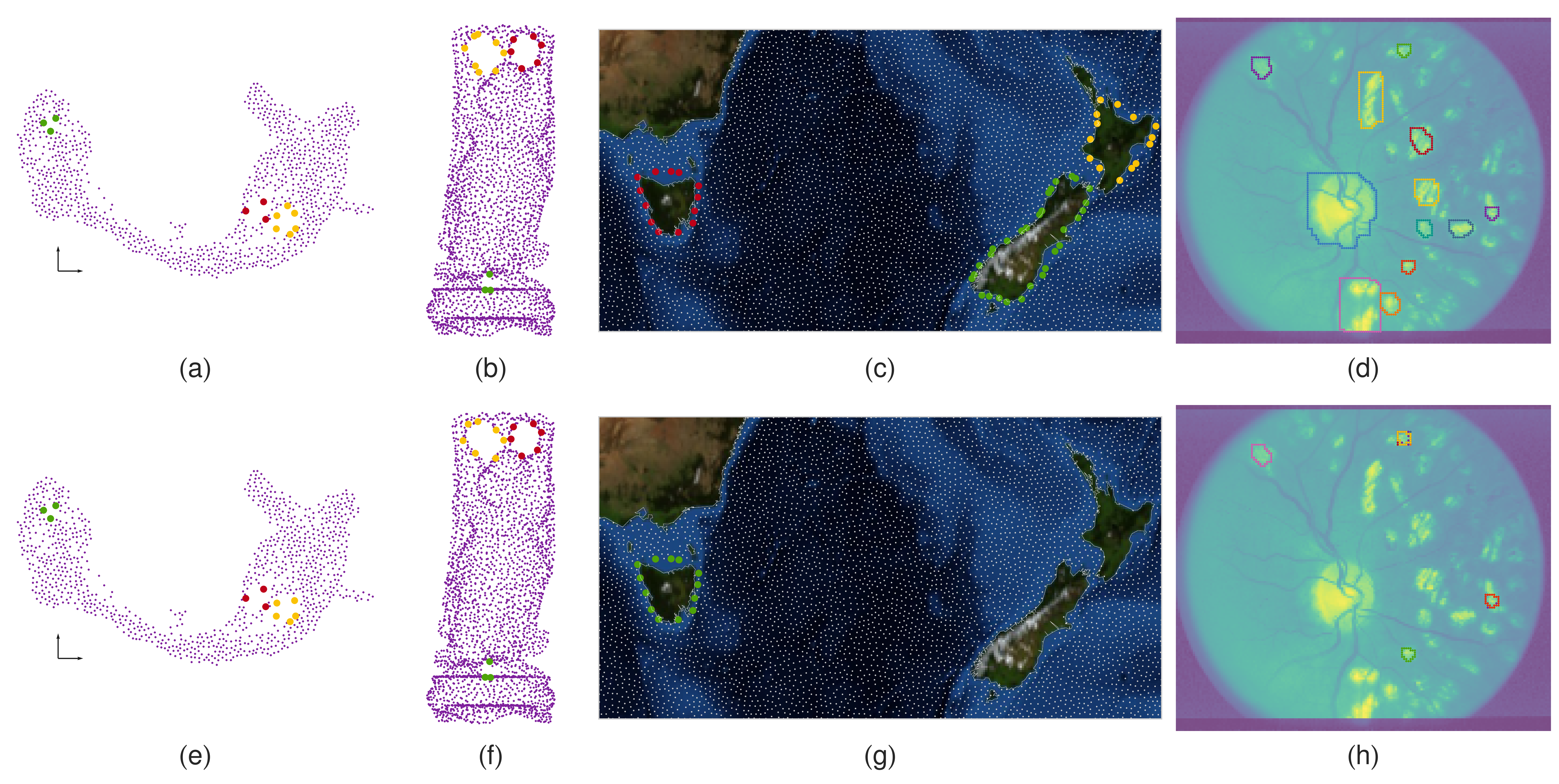}
    \caption{Comparison of the homologous loop detections on $\vZ$ (the first row) and $\vY$ (the second row). The first, the second, the third, and the fourth columns present the results on \datrna, \datbudd, \datisland, and \datretina, respectively. Note that (a)--(d) are identical to Figures \figref{fig:detected-cycles-rna}--\figref{fig:detected-cycles-retina}.}
    \label{fig:real-datasets-comp-coupled-indep}
\end{figure}

Figure \ref{fig:real-datasets-comp-coupled-indep} shows the results of the
shortest homologous loop detection algorithm applied on the coupled
homology embeddings $\vY$ on the real datasets. Note that Figures 
\ref{fig:real-datasets-comp-coupled-indep}a--\ref{fig:real-datasets-comp-coupled-indep}d
are identical to Figures \figref{fig:detected-cycles-rna}--\figref{fig:detected-cycles-retina}; 
they are presented here as comparisons to the loops detected from $\vY$ (the second
row). As shown in Figures \ref{fig:real-datasets-comp-coupled-indep}g
and \ref{fig:real-datasets-comp-coupled-indep}h,
duplicated loops might be extracted if using the coupled embedding $\vY$; these
loops are clearly sub-optimal.

 \section{Pseudocodes}
\label{sec:pseudocodes}

\begin{algorithm}[!htb]
    \setstretch{1.15}
    \SetKwInOut{Input}{Input}
    \SetKwInOut{Output}{Return}
    \SetKwComment{Comment}{$\triangleright $\ }{}
	\Input{$\vZ = [\vz_1,\cdots,\vz_{\beta_1}]$, $V$, $E$, edge distance $\vd$}
	\For{$i = 1,\cdots,\beta_1$}{
		$E_i^+ \gets \{ (s,t): e=(s,t)\in E \text{ and } [\vz_i]_e > 0\}$ \label{algstp:eiplus} \\
		$E_i^- \gets \{ (t,s): e=(s,t)\in E \text{ and } [\vz_i]_e < 0 \}$ \label{algstp:eiminus} \\
		$E_i\gets E_i^+ \cup E_i^-$ \label{algstp:ei} \\
		$G_i \gets (V, E_i)$, with weight of $e\in E_i$ being $[\vd]_e$ \label{algstp:weighted-Gi} \\
		$e_* = (t, s_0) \gets \argmax_{e\in E_i} [\vz_i]_e$ \\
		$[s_0,s_1,\cdots,t]\gets \textsc{Dijkstra}(G_i,\texttt{from=}s_0, \texttt{to=}t)$ \\
		$\cC_i \gets [t, s_0,s_1,\cdots, t]$ 
	}
	\Output{ $\cC_1,\cdots,\cC_{\beta_1}$}
	\caption{Spectral homologous loop detection---an alternative to Algorithm \ref{alg:dijkstra-shortest-loop}}
    \label{alg:alternative-dijkstra-shortest-loop}
\end{algorithm}

\begin{algorithm}[!htb]
    \setstretch{1.15}
    \SetKwInOut{Input}{Input}
    \SetKwInOut{Output}{Return}
    \SetKwComment{Comment}{$\triangleright $\ }{}
    \Input{$\scx_\ell=(\Sigma_0,\cdots,\Sigma_\ell)$, $k$}
    \Comment{Requires $\ell \geq k+1$}
    $\vB_k \gets \textsc{BoundaryMap}(\Sigma_{k-1}, \Sigma_k)$ \Comment{Algorithm \ref{alg:boundary-map-single}} 
    $\vB_k \gets \textsc{BoundaryMap}(\Sigma_{k}, \Sigma_{k+1})$ \\
    
    \Output{Boundary maps $\vB_k$, $\vB_{k+1}$}
    \caption{\textsc{BoundaryMaps}}
    \label{alg:boundary-map}
\end{algorithm}

\begin{algorithm}[!htb]
    \setstretch{1.15}
    \SetKwInOut{Input}{Input}
    \SetKwInOut{Output}{Return}
    \SetKwComment{Comment}{$\triangleright $\ }{}
    \Input{Set of $(k-1)$ and $k$-simplices $\Sigma_{k-1}$, $\Sigma_k$ (or cubes $K_{k-1}$, $K_k$)}
    $\vec{B}_k\gets \vec{0}_{n_{k-1}}\vec{0}_{n_{k}}^\top\in\mathbb R^{n_{k-1}\times n_k}$ \\
    \For{every $\sigma_{k-1}\in \Sigma_{k-1}$}{
        \For{every $\sigma_k \in \Sigma_k$}{
            \uIf{$\sigma_{k-1}$ is a face of $\sigma_k$}{
                $[\vec{B}_k]_{\sigma_{k-1},\sigma_k} \gets \textsc{Orientation}(\sigma_{k-1},\sigma_k)$
            }
            \Else{$[\vec{B}_k]_{\sigma_{k-1}, \sigma_k} \gets 0$}
        }
    }
    \Output{$k$-th boundary map $\vec{B}_k$}
    \caption{\textsc{BoundaryMap}}
    \label{alg:boundary-map-single}
\end{algorithm}

\begin{algorithm}[!htb]
    \setstretch{1.15}
    \SetKwInOut{Input}{Input}
    \SetKwInOut{Output}{Return}
    \SetKwComment{Comment}{$\triangleright $\ }{}
    \Input{Initial point cloud $\tilde\vX\in\rrr^{n'\times D}$, number of furthest points $n$}
    $\vX\gets \emptyset$ \\
    Pick a point $\hat\vx\in\rrr^D$ randomly from $\tilde\vX$ \\
    \For{$i=1,\cdots,n-1$}{
    	$\vX \gets \vX \cup \{ \hat\vx \}$ \Comment{Add $\hat\vx$ to $\vX$}
    	$\tilde\vX \gets \tilde\vX \backslash \{\hat\vx\}$ \Comment{Remove $\hat\vx$ from $\tilde\vX$}
    	$\hat\vx \gets \argmax_{\vx\in \vX}\min_{\tilde\vx \in\tilde\vX} \|\vx - \tilde\vx\|$ \\
    	\Comment{Find the point $\hat\vx$ in $\vX$ that is furthest from $\tilde\vX$}
    }
    \Output{Point cloud $\vX\in\rrr^{n\times D}$}
    \caption{Furthest points sampling}
    \label{alg:furthest-point-sampling}
\end{algorithm}

\end{document}